\def\year{2018}\relax
\documentclass[letterpaper]{article} 
\usepackage{aaai18}  
\usepackage{times}  
\usepackage{helvet}  
\usepackage{courier}  
\usepackage{url}  

\usepackage[utf8]{inputenc} 
\usepackage{hyperref}       
\usepackage{url}            
\usepackage{booktabs}       
\usepackage{amsfonts}       
\usepackage{nicefrac}       
\usepackage{microtype}      

\usepackage{amsmath}

\usepackage{graphicx} 
\usepackage{subfigure} 

\graphicspath{{figures/}}

\usepackage{wrapfig}

\usepackage{algorithm}
\usepackage{algorithmic}

\input{Definitions}

\ifx\proof\undefined
\newenvironment{proof}{\par\noindent{\bf Proof\ }}{\hfill\BlackBox\\[2mm]}
\fi

\ifx\theorem\undefined
\newtheorem{theorem}{Theorem}
\fi
\ifx\example\undefined
\newtheorem{example}{Example}
\fi
\ifx\lemma\undefined
\newtheorem{lemma}[theorem]{Lemma}
\fi

\ifx\corollary\undefined
\newtheorem{corollary}[theorem]{Corollary}
\fi

\ifx\assumption\undefined
\newtheorem{assumption}{Assumption}
\fi

\ifx\definition\undefined
\newtheorem{definition}{Definition}
\fi

\ifx\proposition\undefined
\newtheorem{proposition}[theorem]{Proposition}
\fi

\ifx\remark\undefined
\newtheorem{remark}{Remark}
\fi

\ifx\conjecture\undefined
\newtheorem{conjecture}[theorem]{Conjecture}
\fi

\ifx\factoid\undefined
\newtheorem{factoid}[theorem]{Fact}
\fi

\ifx\axiom\undefined
\newtheorem{axiom}[theorem]{Axiom}
\fi

\newcommand{\RN}[1]{%
	\textup{\lowercase\expandafter{\it \romannumeral#1}}%
}

\begin{document}
%
\title{A Convergence Analysis for A Class of \\Practical Variance-Reduction Stochastic Gradient MCMC}
\author{Changyou Chen$^1$, Wenlin Wang$^2$, Yizhe Zhang$^2$, Qinliang Su$^2$, Lawrence Carin$^2$\\
$^1$University at Buffalo; $^2$Duke University
}
\maketitle
\begin{abstract}
Stochastic gradient Markov Chain Monte Carlo (SG-MCMC) has been developed as a flexible family of scalable Bayesian sampling algorithms. However, there has been little theoretical analysis of the impact of minibatch size to the algorithm's convergence rate. In this paper, we prove that under a limited computational budget/time, a larger minibatch size leads to a faster decrease of the mean squared error bound (thus the fastest one corresponds to using full gradients), which motivates the necessity of variance reduction in SG-MCMC. Consequently, by borrowing ideas from stochastic optimization, we propose a practical variance-reduction technique for SG-MCMC, that is efficient in both computation and storage. We develop theory to prove that our algorithm induces a faster convergence rate than standard SG-MCMC. A number of large-scale experiments, ranging from Bayesian learning of logistic regression to deep neural networks, validate the theory and demonstrate the superiority of the proposed variance-reduction SG-MCMC framework.
\end{abstract}

\noindent 

\section{Introduction}
With the increasing size of datasets of interest to machine learning, stochastic gradient Markov Chain Monte Carlo (SG-MCMC) has been established as an effective tool for large-scale Bayesian learning, with applications in topic modeling \cite{GanCHCC:icml15,LiuZS:NIPS16}, matrix factorization \cite{ChenFG:ICML14,DingFBCSN:NIPS14,SBCG:ICML16}, differential privacy \cite{WangFS:icml15}, Bayesian optimization \cite{SpringenbergKFH:NIPS16} and deep neural networks \cite{PSGLD:AAAI16}. Typically, in each iteration of an SG-MCMC algorithm, a minibatch of data is used to generate the next sample, yielding computational efficiency comparable to stochastic optimization. While a large number of SG-MCMC algorithms have been proposed, their optimal convergence rates generally appear to share the same form, and are typically slower than stochastic gradient descent (SGD) \cite{ChenDC:NIPS15}. The impact of stochastic gradient noise comes from a higher-order term (see Lemma~\ref{lem:biasmse} below), which was omitted in the analysis of \cite{ChenDC:NIPS15}. In other words, current theoretical analysis for SG-MCMC does not consider the impact of minibatch size (corresponding to stochastic gradient noise), making the underlying convergence theory w.r.t.\! minibatch size unclear. Recent work by \cite{DubeyRPSX:nips16} on applying variance reduction in stochastic gradient Langevin dynamics (SGLD) claims to improve the convergence rate of standard SGLD \cite{WellingT:ICML11,TehTV:arxiv14,VollmerZT:arxiv15}. 

The theoretical analysis in \cite{DubeyRPSX:nips16} omits certain aspects of variance reduction in SGLD, that we seek to address here: $\RN{1}$) how does the minibatch size (or equivalently the stochastic gradient noise) affect the convergence rate of an SG-MCMC algorithm? and $\RN{2}$) how can one effectively reduce the stochastic gradient noise in SG-MCMC to improve its convergence rate, from both an algorithmic and a theoretical perspective? For $(\RN{1})$, we provide theoretical results on the convergence rates of SG-MCMC w.r.t.\! minibatch size. For $(\RN{2})$, we propose a practical variance-reduction technique for SG-MCMC, as well as theory to analyze improvements of the corresponding convergence rates. The resulting SG-MCMC algorithm is referred to as variance-reduction SG-MCMC (vrSG-MCMC).

For a clearer description, we first define notation. In a Bayesian model, our goal is typically to evaluate the {\em posterior average} of a test function $\phi(\xb)$,
defined as
$\bar{\phi} \triangleq \int_{\mathcal{X}} \phi(\xb) \rho(\xb) \mathrm{d}\xb$,
where $\rho(\xb)$ is the target posterior distribution with $\xb$ the possibly augmented
model parameters (see Section~\ref{sec:prelim}). Let $\{\xb_{l}\}_{l=1}^L$ be the samples generated from an SG-MCMC algorithm. We use the {\em sample average}, 
$\hat{\phi}_L \triangleq \frac{1}{L} \sum_{l = 1}^L \phi(\xb_{l})$, to approximate $\bar{\phi}$. The corresponding {\em bias} and mean square error ({\em MSE}) are defined as $|\mathbb{E}\hat{\phi}_L - \bar{\phi}|$ and $\mathbb{E}\left(\hat{\phi}_L - \bar{\phi}\right)^2$, respectively. In vrSG-MCMC, unbiased estimations of full gradients are used, leading to the same bias bound as standard SG-MCMC \cite{ChenDC:NIPS15}. As a result, we focus here on analyzing the MSE bound for vrSG-MCMC.

Specifically, we first analyze how minibatch size affects the MSE convergence rate of standard SG-MCMC, summarized in two cases: $\RN{1}$) for a limited computation budget, the optimal MSE bound is achieved when using full gradients in the algorithm; $\RN{2}$) for a large enough computational budget, {\it i.e.}, in a long-run setting, stochastic gradients with minibatches of size one are preferable. This indicates that stochastic gradient noise hurts SG-MCMC at the beginning of the algorithm. While it is computationally infeasible to use full gradients in practice, a remedy to overcome this issue is to use relatively small minibatches with variance reduction techniques to reduce stochastic gradient noise. Consequently, we propose a practical variance-reduction scheme, making SG-MCMC computationally efficient in a big-data setting. Finally, we develop theory to analyze the benefit of the proposed variance-reduction technique and empirically show improvements of vrSG-MCMC over standard SG-MCMC algorithms.

\section{Preliminaries}\label{sec:prelim}


SG-MCMC is a family of scalable Bayesian sampling algorithms, developed recently to generate approximate samples from a posterior distribution $p(\thetab | \Db)$. 
Here $\thetab \in \mathbb{R}^r$ represents a model parameter vector and  $\Db = \{\db_1, \cdots, \db_N\}$ represents the data available to learn the model.
In general, SG-MCMC algorithms are discretized numerical approximations of continuous-time It\^{o} diffusions \cite{ChenDC:NIPS15,MaCF:NIPS15}, which are equipped with stationary distributions coincident with the target posterior distributions. An It\^{o} diffusion is written as
\begin{align}
\mathrm{d}\xb_t &= F(\xb_t)\mathrm{d}t + g(\xb_t)\mathrm{d}\mathcal{\wb}_t~, \label{eq:itodiffusion}
\end{align}
where $\xb \in \mathbb{R}^d$ is the state variable, $t$ is the time index, and $\mathcal{\wb}_t \in \mathbb{R}^d$ is $d$-dimensional Brownian motion. Typically, $\xb \supseteq \thetab$ is an augmentation of the model parameters, so  $r \le d$. Functions $F: \mathbb{R}^d \to \mathbb{R}^d$ and $g: \mathbb{R}^d \rightarrow \mathbb{R}^{d\times d}$ are assumed to satisfy the Lipschitz continuity condition \cite{Ghosh:book11}.

According to \cite{MaCF:NIPS15}, all SG-MCMC algorithms can be formulated by defining appropriate functions $F$ and $g$ in \eqref{eq:itodiffusion}. For example, the stochastic gradient Langevin dynamic (SGLD) model corresponds to $\xb = \thetab$, and $F(\xb_t) = -\nabla_{\thetab} U(\thetab), \hspace{0.1cm} g(\xb_t) = \sqrt{2}\Ib_r$, where $U(\thetab) \triangleq -\log p(\thetab) - \sum_{i=1}^N \log p(\db_i | \thetab)$ denotes the unnormalized negative log-posterior. Similar formula can be defined for other SG-MCMC algorithms, such as stochastic gradient Hamiltonian Monte Carlo (SGHMC) \cite{ChenFG:ICML14} and stochastic gradient thermostats (SGNHT) \cite{DingFBCSN:NIPS14}.

An SG-MCMC algorithm is usually developed by numerically solving the corresponding It\^{o} diffusion and replacing the full gradient $\nabla_\thetab U(\thetab)$ with an unbiased estimate from a minibatch of data $\nabla_\thetab \tilde{U}(\thetab)$ in each iteration. For example, in SGLD, this yields an update equation of $\thetab_l = \thetab_{l-1} - \nabla_\thetab \tilde{U}(\thetab_{l-1}) h_l + \sqrt{2h_l}\zeta_{l}$ for the $l$-th iteration, where $h_l$ is the stepsize, $\zeta_l \sim N(\mathbf{0}, \Ib_r)$. This brings two sources of error into the chain: numerical error (from discretization of the differential equation) and stochastic noise error from use of minibatches. In particular, \cite{ChenDC:NIPS15} proved 
the following bias and MSE bounds for general SG-MCMC algorithms:

\begin{lemma}[\cite{ChenDC:NIPS15}]\label{lem:biasmse}
	Under Assumption~\ref{ass:assumption1} in Appendix~\ref{app:ass}, the bias and MSE 
	of SG-MCMC with a $K$th-order integrator\footnote{The order characterizes the accuracy of a numerical integrator, {\it e.g.}, the Euler method is a 1st-order integrator.} at time $t = hL$ are bounded as:
	{\small\begin{align*}
		\left|\mathbb{E}\hat{\phi}_L - \bar{\phi}\right| &= O\left(\frac{\sum_l \left\|\mathbb{E}\Delta V_l\right\|}{L} + \frac{1}{Lh} + h^K\right) \\
		\mathbb{E}\left(\hat{\phi}_L - \bar{\phi}\right)^2 &= O \left(\frac{\frac{1}{L}\sum_l\mathbb{E}\left\|\Delta V_l\right\|^2}{L} + \frac{1}{Lh} + h^{2K}\right)
		\end{align*}}
\end{lemma}
Here $\Delta V_l \triangleq (\mathcal{L} - \tilde{\mathcal{L}}_l)\phi$, 
where $\mathcal{L}$ is the infinitesimal generator of the It\^{o} 
diffusion \eqref{eq:itodiffusion} defined as $\mathcal{L}f(\xb_t) = \rbr{F(\xb_t) \cdot \nabla_{\xb} + \frac{1}{2}\left(g(\xb_t) g(\xb_t)^T\right)\!:\! \nabla_{\xb}\!\nabla^T_{\xb}} f(\xb_t)$, for any compactly supported twice differentiable function $f: \mathbb{R}^d \rightarrow \mathbb{R}$. $\ab\cdot \bb \triangleq \ab^T\bb$ for two vectors $\ab$ and $\bb$, $\Ab\!:\!\Bb\triangleq \mbox{tr}\{\Ab^T\Bb\}$ for two matrices $\Ab$ and $\Bb$. $\|\cdot\|$ is defined as the standard {\em operator norm} acting on the space of bounded functions, {\it e.g.}, $\|f\| \triangleq \sup_{\xb}f(\xb)$ for a function $f$. $\tilde{\mathcal{L}}_l$ is the same as $\mathcal{L}$ except for the substitution of the stochastic gradient $\nabla \tilde{U}_{l}(\thetab)$ for the full gradient due to the usage of a stochastic gradient in the $l$-th iteration. By substituting the definition of $\Delta V_l$ and $\mathcal{L}$, typically we have $\Delta V_l = (\nabla_{\thetab}U_l(\thetab) - \nabla_{\thetab}\tilde{U}_l(\thetab))\!\cdot\!\nabla\phi$.

By using an unbiased estimate of the true gradient, the term $\mathbb{E}\Delta V_l$ in the bias bound in Lemma~\ref{lem:biasmse} vanishes, indicating that stochastic gradients (or equivalently minibatch size) only affect the MSE bound. Consequently, we focus on improving the MSE bound with the proposed variance-reduction SG-MCMC framework.

\section{Practical Variance-Reduction SG-MCMC}
We first motivate the necessity of variance reduction in SG-MCMC, by analyzing how minibatch size affects the MSE bound. A practical variance reduction scheme is then proposed, which is efficient from both computational and storage perspectives. Comparison with existing variance-reduction SG-MCMC approaches is also highlighted. Previous research has revealed that the convergence of diffusion-based MCMC scales at an order of $O(d^{1/3})$ w.r.t.\! dimension $d$ \cite{DurmusRVZ:arxiv16}. For the interest of SG-MCMC, we following standard analysis \cite{VollmerZT:arxiv15} and do not consider the impact of $d$ in our analysis.

\subsection{The necessity of variance reduction: a theoretical perspective}\label{sec:necessary_vr}
It is clear from Lemma~\ref{lem:biasmse} that the variance of noisy stochastic gradients plays an important role in the MSE bound of an SG-MCMC algorithm. What is unclear is how exactly minibatch size affects the convergence rate. Intuitively, minibatch size appears to play the following roles in SG-MCMC: $\RN{1}$) smaller minibatch sizes introduce larger variance into stochastic gradients; $\RN{2}$) smaller minibatch sizes allow an algorithm to run faster (thus more samples can be obtained in a given amount of computation time). To balance the two effects, in addition to using the standard assumptions for SG-MCMC (which basically requires the coefficients of It\^{o} diffusions to be smooth and bounded, and is deferred to Assumption~\ref{ass:assumption1} in the Appendix), we assume that the algorithms with different minibatch sizes all run for a fixed computational time/budget $T$ in the analysis, as stated in Assumption~\ref{ass:budget}.

\begin{assumption}\label{ass:budget}
	For a fair comparison, all SG-MCMC algorithms with different minibatch sizes are assumed to run for a fixed amount of computation time/budget $T$. Further, we assume that $T$ linearly depends on the minibatch size $n$ and the sample size $L$, {\it i.e.}, $T \propto nL$.
\end{assumption}

For simplicity, we rewrite the gradient of the log-likelihood for data $\db_i$ in the $l$-th iteration as: $\alphab_{li} = \nabla_{\thetab}\log p(\db_i | \thetab_{l})$. We first derive the following lemma about the property of $\{\alphab_{li}\}$, which is useful in the subsequent developments, {\it e.g.}, to guarantee a positive bound in Theorem~\ref{theo:mse_batch} and an improved bound for the proposed vrSG-MCMC (Theorem~\ref{theo:main}).

\begin{lemma}\label{lem:exp_alpha}
	Under Assumption~\ref{ass:assumption1}, given $\thetab_{l}$ in the $l$-th iteration,
	$\Gamma_l \triangleq \frac{1}{N^2}\sum_{i=1}^N\sum_{j=1}^N \mathbb{E}\left[\alphab_{li}^T\alphab_{lj}\right] - \frac{\sum_{i\neq j}\mathbb{E}\alphab_{li}^T\alphab_{lj}}{N(N-1)} \geq 0$, 
	where the expectation is taken over the randomness of an SG-MCMC algorithm\footnote{The same meaning goes for other expectations in the paper if not explicitly specified.}.
\end{lemma}

We next generalize Lemma~\ref{lem:biasmse} by incorporating the minibatch size $n$ into the MSE bound. The basic idea in our derivation is to associate with each data $\db_i$ a binary random variable, $z_i$, to indicate whether data $\db_i$ is included in the current minibatch or not. These $\{z_i\}$ depend on each other such that $\sum_{i=1}^Nz_i = n$ in order to guarantee minibatches of size $n$. Consequently, the stochastic gradient in the $l$-th iteration can be rewritten as: $\nabla_{\thetab}\tilde{U}_l(\thetab) = -\nabla_{\thetab}\log p(\thetab_l) - \frac{N}{n}\sum_{i=1}^{N}\nabla_{\thetab}\log p(\db_i|\thetab_l)z_i$. 
Substituting the above gradient formula into the proof of standard SG-MCMC \cite{ChenDC:NIPS15} and further summing out $\{z_i\}$ results in an alternative MSE bound for SG-MCMC, stated in Theorem~\ref{theo:mse_batch}. In the analysis, we assume to use a 1st-order numerical integrator for simplicity, {\it e.g.} the Euler method, though the results generalize to $K$th-order integrators easily.

\begin{theorem}\label{theo:mse_batch}
	Under Assumption~\ref{ass:assumption1}, let the minibatch size of an SG-MCMC be $n$, $\Gamma_M \triangleq \max_l \Gamma_l$. The finite-time MSE is bounded, for a constant $C$ independent of $\{h, L, n\}$, as:
	{\small\begin{align*}
		\mathbb{E}&\left(\hat{\phi}_L - \bar{\phi}\right)^2 
		\leq C\left(\frac{2(N-n)N^2\Gamma_M}{nL} + \frac{1}{Lh} + h^2\right)~.
		\end{align*}}
\end{theorem}
Theorem~\ref{theo:mse_batch} represents the bound in terms of minibatch size $n$ and sample size $L$. Note in our finite-time setting, $L$ and $N$ are considered to be constants. Consequently, $\Gamma_M$ is also a bounded constant in our analysis. To bring in the computational budget $T$, based on Assumption~\ref{ass:budget}, {\it e.g.}, $T \propto nL$, the optimal MSE bound w.r.t.\! stepsize $h$ in Theorem~\ref{theo:mse_batch} can be written as: $\mathbb{E}\left(\hat{\phi}_L - \bar{\phi}\right)^2 = O\left(\frac{(N-n)N^2 \Gamma_M}{T} + \frac{n^{2/3}}{T^{2/3}}\right)$. After further optimizing the bound w.r.t.\! $n$ by setting the derivative of the above MSE bound to zero, the optimal minibatch size can be written as $n = O\left(\frac{8T}{27N^6\Gamma_M^3}\right)$. To guarantee this bound for $n$ to be finite and integers, it is required that the computational budget $T$ to scale at the order of $O(N^6)$ when varying $N$. When considering both $T$ and $N$ as impact factors, the optimal $n$ becomes more interesting, and is concluded in Corollary~\ref{cor:optimalT}\footnote{Note we only have that $T = C_1nL$ for some unknown constant $C_1$, {\it i.e.}, the specific value of $T$ is unknown.}. 
\begin{corollary}\label{cor:optimalT}
	Under Assumption~\ref{ass:budget} and \ref{ass:assumption1}, we have three cases of optimal minibatch sizes, each corresponding to different levels of computational budget.
	\begin{itemize}
		\item[1)] When the computational budget is small, {\it e.g.}, $T < O\left(\frac{27}{8}\Gamma_M^3 N^6\right)$, the optimal MSE bound is decreasing w.r.t.\! $n$ in range $[1, N]$. The minimum MSE bound is achieved at $n = N$. 
		\item[2)] When the computational budget is large, {\it e.g.}, $T > O\left(\frac{27}{8}\Gamma_M^3 N^7\right)$, the optimal MSE bound is increasing w.r.t.\! $n$ in range $[1, N]$. The minimum MSE bound is achieved at $n = 1$.
		\item[3)] When the computational budget is in between the above two cases, the optimal MSE bound first increases then decreases w.r.t.\! $n$ in range $[1, N]$. The optimal MSE bound is obtained either at $n = 1$ or at $n = N$, depending on $(N, T, \Gamma_M)$.
	\end{itemize}
\end{corollary}

In many machine learning applications, the computational budget is limited, leading the algorithm to the first case of Corollary~\ref{cor:optimalT}, {\it i.e.}, {\small$T < O\left(\frac{27}{8}\Gamma_M^3 N^6\right)$}. According to Corollary~\ref{cor:optimalT}, processing full data ({\it i.e.}, no minibatch) is required to achieve the optimal MSE bound, which is computationally infeasible when $N$ is large (which motivated use of minibatches in the first place). A practical way to overcome this is to use small minibatches and adopt variance-reduction techniques to reduce the stochastic gradient noise.

\subsection{A practical variance reduction algorithm}

For practical use, we require that a variance-reduction method should achieve both computational and storage efficiency. While variance reduction has been studied extensively in stochastic optimization, it is applied much less often in SG-MCMC. In this section we propose a vrSG-MCMC algorithm, a simple extension of the algorithm in \cite{DubeyRPSX:nips16}, but is more computationally practical in large-scale applications. A convergence theory is also developed in Section~\ref{sec:theory}.

The proposed vrSG-MCMC is illustrated in Algorithm~\ref{alg:pvrsgmcmc}. Similar to stochastic optimization \cite{SchmidtRB:MP16}, the idea of variance reduction is to balance the gradient noise with a less-noisy {\em old gradient}, {\it i.e.}, a stochastic gradient is calculated based on a previous sample, as well as using a larger minibatch than that of the {\em current stochastic gradient}, resulting in a less noisy estimation. In each iteration of our algorithm, an unbiased stochastic gradient is obtained by combining the above two versions of gradients in an appropriate way (see $g_{l+1}$ in Algorithm~\ref{alg:pvrsgmcmc}). Such a construction of stochastic gradients essentially inherits a low variance with theoretical guarantees (detailed in Section~\ref{sec:theory}). In Algorithm~\ref{alg:pvrsgmcmc}, the whole parameter $\xb$ is decomposed into the model parameter $\thetab$ and the remaining algorithm-specific parameter $\taub$, {\it e.g.}, the momentum parameter. The expression ``$\thetab \leftarrow \xb$'' means assigning the corresponding model parameter from $\xb$ to $\thetab$. The {old gradient} is denoted as $\tilde{g}$, calculated with a minibatch of size $n_1$. The {\em current stochastic gradient} is calculated on a minibatch of size $n_2 < n_1$. We use $\xb_{l+1} = \mbox{NextS}\left(\xb_{l}, g_{l+1}, h_l \right)$ to denote a function which generates the next sample $\xb_{l+1}$ with an SG-MCMC algorithm, based on the current sample $\xb_{l}$, input stochastic gradient $g_{l+1}$, and step size $h_l$.

\begin{algorithm}[H]
	\caption{Practical Variance-Reduction SG-MCMC.}\label{alg:pvrsgmcmc}
	\begin{algorithmic}
		\STATE {\bf Input:} $\bar{\xb} = \xb_{0} = (\thetab_0, \taub_0) \in \mathbb{R}^d$, minibatch sizes $(n_1, n_2)$ such that $n_1 > n_2$, update interval $m$, total iterations $L$, stepsize $\{h_l\}_{l=1}^L$
		\STATE {\bf Output:} approximate samples $\{\xb_{l}\}_{l=1}^L$
		\FOR {$l = 0$ to $L-1$}
		\IF {$(l \mbox{ mod } m) = 0$}
		\STATE Sample w/t replacement $\{\pi_i\}_{i=1}^{n_1} \subseteq \{1, \cdots, N\}$;
		\STATE $\bar{\xb} = \xb_l$; ~~~$\tilde{\thetab}_l \leftarrow \bar{\xb}$;
		\STATE $\tilde{g} = \frac{N}{n_1}\sum_{i\in\pi}\nabla_{\thetab}\log p(\db_i|\tilde{\thetab}_l)$;
		\ENDIF
		\STATE $\thetab_l \leftarrow \xb_l$; ~~~$\tilde{\thetab}_l \leftarrow \bar{\xb}$;
		\STATE Sample w/t replacement $\{\tilde{\pi}_i\}_{i=1}^{n_2} \subseteq \{1, \cdots, N\}$;
		\STATE $g_{l+1} = \tilde{g} + \nabla_{\thetab}\log p(\thetab_{l}) + \frac{N}{n_2} \sum_{i\in\tilde{\pi}}\left(\nabla_{\thetab}\log p(\db_i|\thetab_l) - \nabla_{\thetab}\log p(\db_i|\tilde{\thetab}_l) \right)$;
		\STATE $\xb_{l+1} = \mbox{NextS}\left(\xb_{l}, g_{l+1}, h_{l+1} \right)$;
		\ENDFOR
	\end{algorithmic}
\end{algorithm}

One should note that existing variance-reduction algorithms, {\it e.g.}\! \cite{JohnsonZ:NIPS13}, use a similar concept to construct low-variance gradients. However, most algorithms use the whole training data to compute $\tilde{g}$ in Algorithm~\ref{alg:pvrsgmcmc}, which is computationally infeasible in large-scale settings. Moreover, we note that like in stochastic optimization \cite{ReddiHSPS:ICML16,AllenZhuH:ICML16}, instead of using a single parameter sample to compute $\tilde{g}$, similar methods can be adopted to compute $\tilde{g}$ based on an average of {\em old} parameter samples. The theoretical analysis can be readily adopted for such cases, which is omitted here for simplicity. More references are discussed in Section~\ref{sec:related}.

\subsection{Comparison with existing variance-reduction SG-MCMC algorithms}\label{sec:comparison}

The most related variance-reduction SG-MCMC algorithm we are aware of is a recent work on variance-reduction SGLD (SVRG-LD) \cite{DubeyRPSX:nips16}. SVRG-LD shares a similar flavor to our scheme from the algorithmic perspective, except that when calculating the old gradient $\tilde{g}$, the whole training data set is used in SVRG-LD. As mentioned above, this brings a computational challenge for large-scale learning. Although the problem is mitigated by using a moving average estimation of the stochastic gradient, this scheme does not match their theory. A more distinctive advantage of vrSG-MCMC over SVRG-LD \cite{DubeyRPSX:nips16} is in terms of theoretical analysis. Concerning SVRG-LD, $\RN{1}$) the authors did not show theoretically in which case variance reduction is useful in SGLD, and $\RN{2}$) it is not clear in their theory whether SVRG-LD is able to speed up the convergence rate compared to standard SGLD.
Specifically, the MSE of SVRG-LD was shown to be bounded by $O\left(\frac{N^2\min\{2\sigma^2, m^2(D^2h^2\sigma^2+hd)\}}{nL} + \frac{1}{Lh} + h^2\right)$, compared to $O\left(\frac{N^2\sigma^2}{nL} + \frac{1}{Lh} + h^2\right)$ for SGLD, where $(d, D, \sigma)$ are constants. By inspecting the above bounds, it is not clear whether SVRG-LD improves SGLD because the two bounds are not directly comparable\footnote{The first term in the $\min$ of the SVRG-LD bound is strictly larger than the first term of the SGLD bound (if the term $2\sigma^2$ is used in the ``min''), making the bounds not easily compared.}. More detailed explanations are provided in Appendix~\ref{app:discuss}.

\subsection{Convergence rate}\label{sec:theory}

We derive convergence bounds for Algorithm~\ref{alg:pvrsgmcmc} and analyze the improvement of vrSG-MCMC over the corresponding standard SG-MCMC. Using a similar approach as in Section~\ref{sec:necessary_vr}, we first introduce additional binary random variables, $\{b_i\}_{i=1}^N$, to indicate which data points are included in calculating the {\em old gradient} $\tilde{g}$ in Algorithm~\ref{alg:pvrsgmcmc}. This results in the expression for the stochastic gradient used in the $l$-th iteration: $\nabla_{\thetab}\tilde{U}(\thetab_{l}) = \frac{N}{n_2}\sum_{i=1}^N\left(\nabla_{\thetab}\log p(\db_{i}|\thetab_{l}) - \nabla_{\thetab}\log p(\db_{i}|\tilde{\thetab}_{l})\right)z_i + \frac{N}{n_1}\sum_{i=1}^N\sum_{i=1}^N\nabla_{\thetab}\log p(\db_{i}|\tilde{\thetab}_{l}) b_i$. It is easy to verify that the above stochastic gradient is an unbiased estimation of the true gradient in the $l$-th iteration (see Appendix~\ref{app:theory}).

In order to see how Algorithm~\ref{alg:pvrsgmcmc} reduces the variance of stochastic gradients, from Lemma~\ref{lem:biasmse}, it suffices to study $\Delta V_l$, as the minibatch size only impacts this term. For notational simplicity, similar to the $\alphab_{li}$ defined in Section~\ref{sec:necessary_vr}, we denote $\betab_{li} \triangleq \nabla_{\thetab}\log p(\db_i | \tilde{\thetab}_{l})$, which is similar to $\alphab_{li}$ but evaluated on the old parameter $\tilde{\thetab}_l$. Intuitively, since the {\em old gradient} $\tilde{g}$ is calculated from $\betab$ to balance the stochastic gradient noise (calculated from $\alphab$), $\alphab$ and $\betab$ are expected to be close to each other. Lemma~\ref{lem:alpha_beta} formulates the intuition, a key result in proving our main theorem, where we only consider the update interval $m$ and stepsize $h$ as factors. In the lemma below, following \cite{ChenDLZC:NIPS16} (Assumption~1), we further assume the gradient function $\nabla_{\thetab}U(\thetab)$ to be Lipschitz.

\begin{lemma}\label{lem:alpha_beta}
	Under Assumption~\ref{ass:assumption1} and assume $\nabla_{\thetab}U(\thetab)$ to be Lipschitz (Assumption~1 in \cite{ChenDLZC:NIPS16}), $\alphab_{li}$ and $\betab_{li}$ are close to each other in expectation, {\em i.e.}, $\mathbb{E}\alphab_{li} = \mathbb{E}\betab_{li} + O(mh)$.
\end{lemma}
In the Appendix, we further simplify $\mathbb{E}\|\Delta V_l\|^2$ in the MSE bound by decomposing it into several terms. Finally, we arrive at our main theorem for the proposed vrSG-MCMC framework.
\begin{theorem}\label{theo:main}
	Under the setting of Lemma~\ref{lem:alpha_beta}, let $A_M \triangleq \max_l A_l$, and $A_l = \left(\frac{N}{n_2} - 1\right)\sum_{ij}\mathbb{E}\alphab_{li}^T\alphab_{lj} - 2\frac{N(N-n_2)}{n_2(N-1)}\sum_{i<j}\mathbb{E}\alphab_{li}^T\alphab_{lj}$. The MSE of vrSG-MCMC with a $K$th-order integrator is bounded as:
	{\begin{align*}
		\mathbb{E}\left(\hat{\phi}_L - \bar{\phi}\right)^2 = O \left(\frac{A_M}{L} + \frac{1}{Lh} + h^{2K} + \frac{mh}{L} - \frac{\lambda_M}{L}\right)~,
		\end{align*}}
	where $\lambda_M = \min_l \lambda_l$, and $\lambda_l \triangleq \left(\frac{N}{n_1} - \frac{N}{n_2}\right)\sum_{ij}\mathbb{E}\betab_{li}^T\betab_{li} - 2\left(\frac{N(N-n_2)}{n_2(N-1)} - \frac{N(N-n_1)}{n_1(N-1)}\right)\\\sum_{i<j}\mathbb{E}\betab_{li}^T\betab_{lj}$.
	Furthermore, we have $\lambda_l > 0$ for $\forall l$, so that $\lambda_M > 0$.
\end{theorem}

Note that for a fixed $m$, $\frac{mh}{L}$ in the above bound is a high-order term relative to $\frac{1}{Lh}$. As a result, the MSE is bounded by $O \left(\frac{A_M}{L} + \frac{1}{Lh} + h^{2K} - \frac{\lambda_M}{L}\right)$. Because the MSE of standard SG-MCMC is bounded by $O \left(\frac{A_M}{L} + \frac{1}{Lh} + h^{2K}\right)$ (see Appendix~\ref{app:theory}) and $\lambda_M > 0$ from Theorem~\ref{theo:main}, we conclude that vrSG-MCMC induces a lower MSE bound compared to the corresponding SG-MCMC algorithm, with an improvement of $O\left(\frac{\lambda_M}{L}\right)$.

It is worth noting that in Algorithm~\ref{alg:pvrsgmcmc}, the minibatch for calculating the old gradient $\tilde{g}$ is required to be larger than that for calculating the current stochastic gradient, {\it i.e.}, $n_1 > n_2$. Otherwise, $\lambda_l$ in Theorem~\ref{theo:main} would become negative, leading to an increased MSE bound compared to standard SG-MCMC. This matches the intuition that old gradients need to be more accurate (thus with larger minibatches) than current stochastic gradients in order to reduce the stochastic gradient noise.

\begin{remark}
	In the special case of \cite{DubeyRPSX:nips16} where $n_1 = N$ for SGLD, Theorem~\ref{theo:main} gives a MSE bound of $O \left(\frac{A_M}{L} + \frac{1}{Lh} + h^{2} + \frac{mh}{L} - \frac{\max_l\lambda_l}{L}\right)$, with $\lambda_l = \left(\frac{N}{n_2} - 1\right)\sum_{ij}\mathbb{E}\betab_{li}^T\betab_{li} - \frac{2N(N-n_2)}{n_2(N-1)}\sum_{i<j}\mathbb{E}\betab_{li}^T\betab_{lj}$. According to Lemma~\ref{lem:exp_alpha}, $\lambda_l$ is also positive, thus leading to a reduced MSE bound. However, the bound is not necessarily better than that of vrSG-MCMC, where a minibatch is used instead of the whole data set to calculate $\tilde{g}$, leading to a significant decrease of computational time.
\end{remark}

\begin{remark}
	Following Corollary~\ref{cor:optimalT}, Theorem~\ref{theo:main} can also be formulated in terms of the computational budget $T$. Specifically, according to Algorithm~\ref{alg:pvrsgmcmc}, the computational budget $T$ would be proportional to $\frac{n_1}{m} + n_2$. Substituting this into the MSE bound of Theorem~\ref{theo:main} gives a reformulated bound of $O\left(\frac{\left(A_M + mh - \lambda_M\right)\left(\frac{n_1}{m} + n_2\right)}{T} + \frac{1}{Lh} + h^2\right)$. The optimal MSE w.r.t.\! $n_1$ and $n_2$ would be complicated since both $A_M$ and $\lambda_M$ depend on $n_1$ and $n_2$. We omit the details here for simplicity. Nevertheless, our experiments indicate that our algorithm always improves standard SG-MCMC algorithms for the same computational time.
\end{remark}

\section{Related Work}\label{sec:related}

Variance reduction was first introduced in stochastic optimization, which quickly became a popular research topic and has been actively developed in recent years. 
\cite{SchmidtRB:arxiv13,SchmidtRB:MP16} introduced perhaps the first variance reduction algorithm, called stochastic average gradient (SAG), where historical gradients are stored and continuously updated in each iteration. Later, stochastic variance reduction gradient (SVRG) was developed to reduce the storage bottleneck of SAG, at the cost of an increased computational time \cite{JohnsonZ:NIPS13,ZhangMJ:NIPS13}. \cite{DefazioBL:NIPS14} combined ideas of SAG and SVRG and proposed the SAGA algorithm, which improves SAG by using a better and unbiased stochastic-gradient estimation.

Variance reduction algorithms were first designed for convex optimization problems, followed by a number of recent works extending the techniques for non-convex optimization \cite{ReddiHSPS:ICML16,ReddiSPS:NIPS16,AllenZhuH:ICML16,AllenZhuRQY:ICML16}, as well as for distributed learning \cite{ReddiHSPS:NIPS15}. All these algorithms are mostly based on SVRG and are similar in algorithmic form, but differ in the techniques for proving the rigorous theoretical results.

For scalable Bayesian sampling with SG-MCMC, however, this topic has been studied little until a recent work on variance reduction for SGLD \cite{DubeyRPSX:nips16}. In this work, the authors adapted the SAG and SVRG ideas to SGLD. Although they provided corresponding convergence results, some fundamental problems, such as how minibatch size affects the convergence rate, were not fully studied. Furthermore, their algorithms suffer from an either high computational or storage cost in a big-data setting, because the whole data set needs to be accessed frequently.

To reduce the computational cost of SVRG-based algorithms,
the idea of using a minibatch of data to calculate the {\em old gradient} (corresponding to the $\tilde{g}$ in Algorithm~\ref{alg:pvrsgmcmc}) has also been studied in stochastic optimization.
Representative works include, but are not limited to \cite{HarikandehAVSKS:NIPS15,FrostigGKS:COLT15,ShahAKS:arxiv16,LeiJ:NIPS16,LianWL:AISTATS17}. The proposed approach adopts similar ideas, with the following main differences: $\RN{1}$) Our algorithm represents the first work for large-scalable Bayesian sampling with a practical (computationally cheap) variance reduction technique; $\RN{2}$) the techniques used here for analysis are different and appear to be simpler than those used for stochastic optimization; $\RN{3}$) our theory addresses fundamental questions for variance reduction in SG-MCMC, such as those raised in the Introduction.

\section{Experiments}

\subsection{A synthetic experiment}
We first test the conclusion of the long-run setting in Corollary~\ref{cor:optimalT}, which indicates that vrSG-MCMC with minibatches of size 1 achieve the optimal MSE bound. To make the algorithm go into the long-run setting regime as sufficient as possible, we test vrSG-MCMC on a simple Gaussian model, which runs very fast so that a little actual walk-clock time is regarded as a large computational budget. The model is defined as: $x_i \sim \mathcal{N}(\theta, 1), \theta \sim \mathcal{N}(0, 1)$. We generate $N = 1000$ data samples $\{x_i\}$, and calculate the the MSE for minibatch sizes of
$n = 1, 10, 100$. The test function is $\phi(\theta) = \theta^2$. The results are
ploted in Figure~\ref{fig:gau}. We can see from the figure that $n=1$ achieves the
lowese MSE, consistent with the theory (Corollary~\ref{cor:optimalT}).
\begin{figure}
	\centering
	\includegraphics[width=0.6\linewidth]{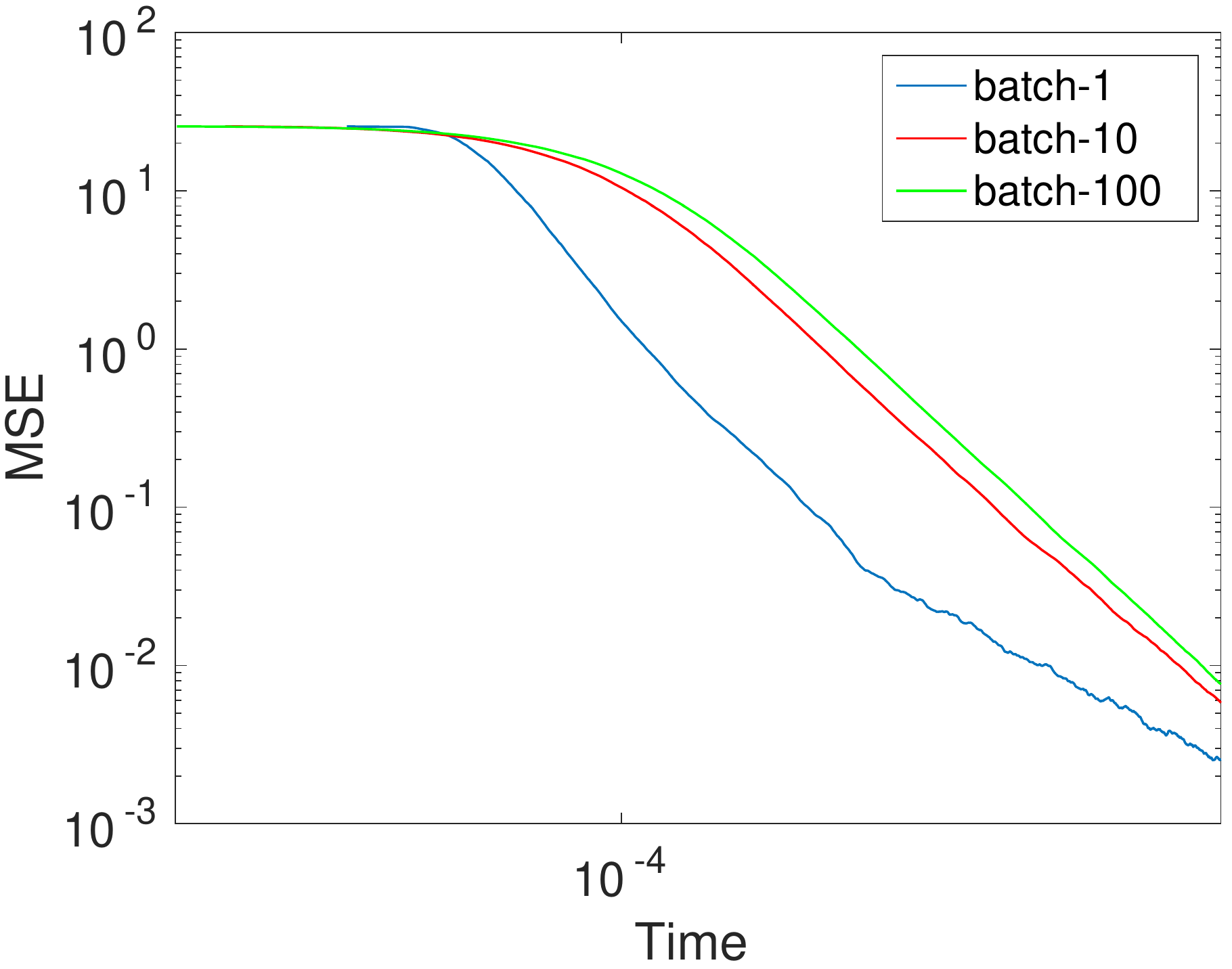}
	\vspace{-0.5cm}
	\caption{MSE vs. wall-clock time for different minibatch sizes.} \label{fig:gau}
	\vspace{-0.5cm}
\end{figure}

\subsection{Applications on deep neural networks}
We apply the proposed vrSG-MCMC framework to Bayesian learning of deep neural networks, including the multilayer perceptron (MLP), convolutional neural network (CNN), and recurrent neural network (RNN). The latter two have not been empirically evaluated in previous work. Experiments with Bayesian logistic regression are given in Appendix~\ref{app:exp}. In the experiments, we are interested in modeling weight uncertainty of neural networks, which is an important topic and has been well studied \cite{HernaandezLobatoA:icml15,BlundellCKW:icml15,PSGLD:AAAI16,LouizosW:ICML16}. We achieve this goal by applying priors to the weights (in our case, we use simple isotropic Gaussian priors) and performing posterior sampling with vrSG-MCMC or SG-MCMC. We implement vrSG-MCMC based on SGLD, and compare it to the standard SGLD and SVRG-LD \cite{DubeyRPSX:nips16} in our experiments\footnote{The SAGA-LD algorithm in \cite{DubeyRPSX:nips16} is not compared here because it is too storage-expensive thus is not fair.}. For this reason, comparisons to other optimization-based methods such as the maximum likelihood are not considered. For simplicity, we set the update interval for the {\em old gradient} $\tilde{g}$ in Algorithm~\ref{alg:pvrsgmcmc} to $m=10$. For all the experiments, the minibatch sizes for vrSG-MCMC are set to $n_1 = 100$ and $n_2 = 10$. To be fair, this corresponds to a minibatch size of $n = 10$ in SGLD and SVRG-LD. Sensitivity of model performance w.r.t.\! minibatch size $n_1$ is tested in Section~\ref{sec:sensity}. For a fair comparison, following convention \cite{AllenZhuH:ICML16,DubeyRPSX:nips16}, we plot the number of data passes versus error in the figures\footnote{Since true posterior averages are infeasible, we plot sample averages in terms of accuracy/loss.}. Results on the number of data passes versus loss are given in the Appendix. In addition, we use fixed stepsizes in our algorithm for all except for the ResNet model specified below. Following relevant literature \cite{JohnsonZ:NIPS13,DubeyRPSX:nips16}, we tune the stepsizes and plot the best results for all the algorithms to ensure fairness. Note in our Bayesian setup, it is enough to run an algorithm for once since the uncertainty is encoded in the samples.

\subsection{Multilayer perceptron}\label{sec:exp_mlp}

We follow conventional settings \cite{ReddiHSPS:ICML16,AllenZhuH:ICML16} and use a single-layer MLP with 100 hidden units, using the {\em sigmoid} activation function as the nonlinear transformation. We test the MLP on the MNIST and CIFAR-10 datasets. The stepsizes for both vrSG-MCMC and SGLD are set to 0.25 and 0.01 in the two datasets, respectively. Figure~\ref{fig:fnn} plots the number of passes through the data versus test error/loss. Results on the training datasets, including training results for the CNN and RNN-based deep learning models described below, are provided in Appendix~\ref{app:exp}. It is clear that vrSG-MCMC leads to a much faster convergence speed than SGLD, resulting in much lower test errors and loss at the end, especially on the CIFAR-10 dataset. SVRG-LD, though it leads to potential lower errors/loss, converges slower than vrSG-MCMC, due to the high computational cost in calculating the {\em old gradient} $\tilde{g}$. As a result, we do not compare vrSG-MCMC with SVRG-LD in the remaining experiments.

\begin{figure}[ht]
	\begin{center}
		\begin{minipage}{0.49\linewidth}
			\includegraphics[width=\columnwidth]{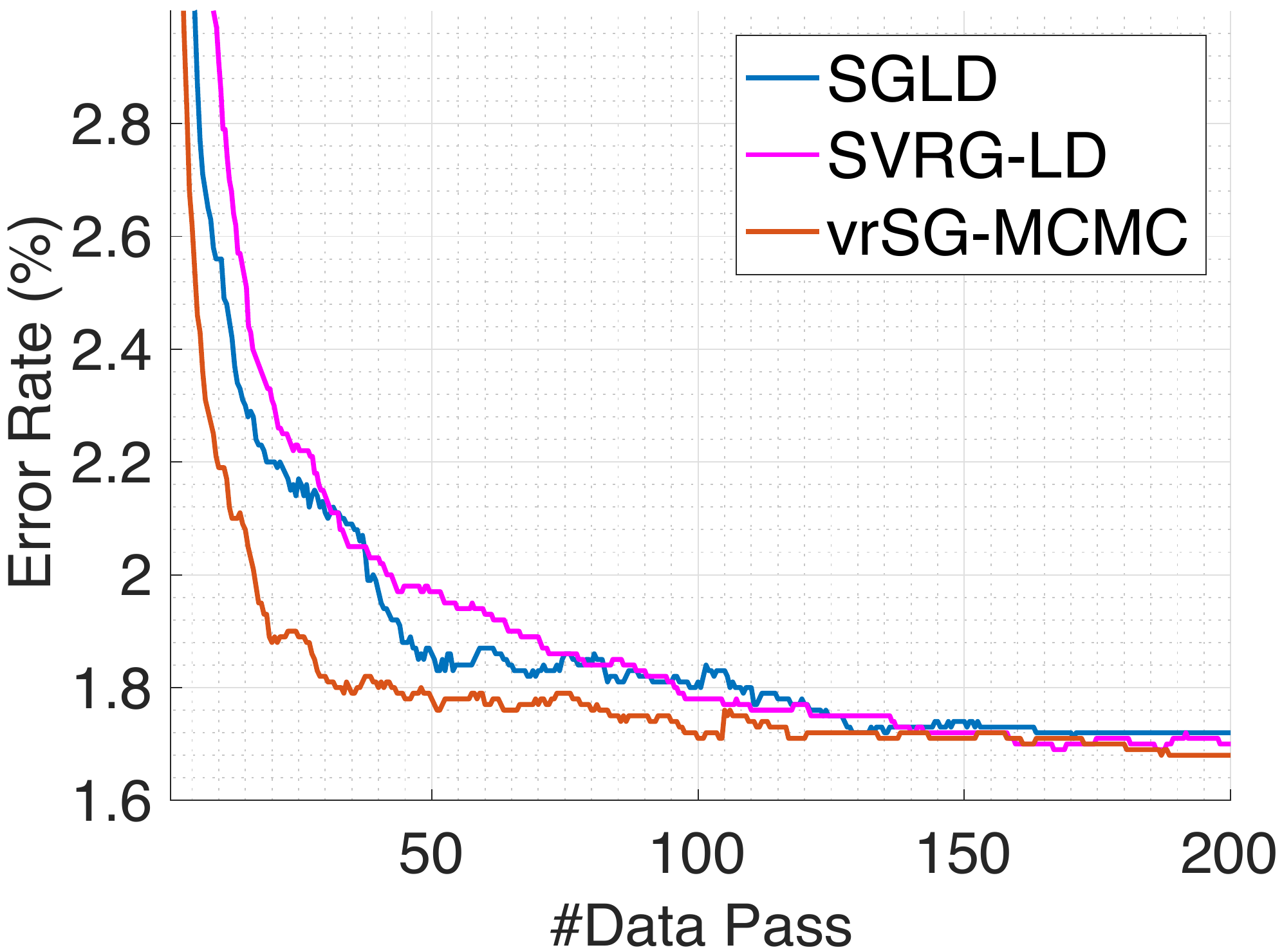}
		\end{minipage}
		\begin{minipage}{0.49\linewidth}
			\includegraphics[width=\columnwidth]{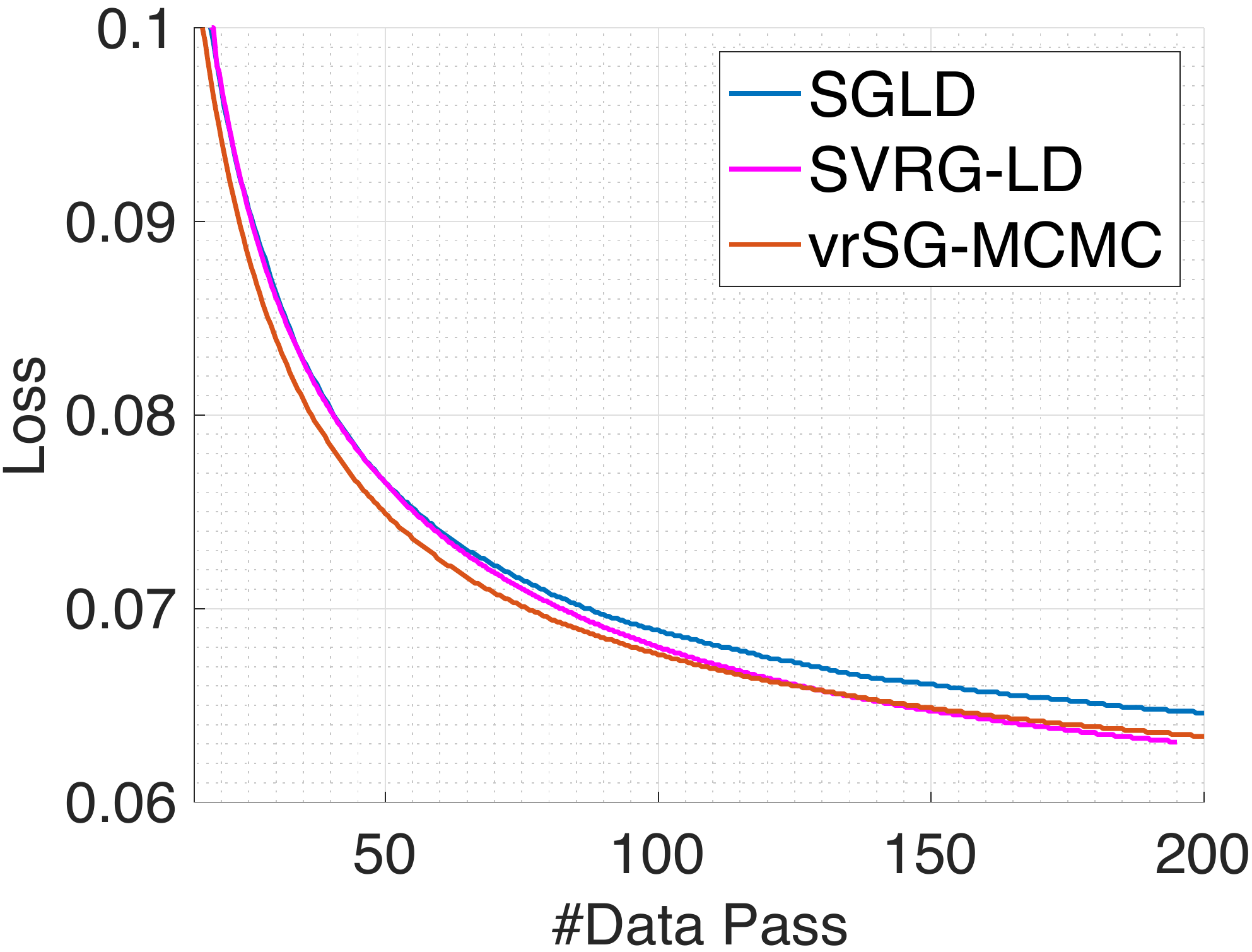}
		\end{minipage}
		\begin{minipage}{0.49\linewidth}
			\includegraphics[width=\columnwidth]{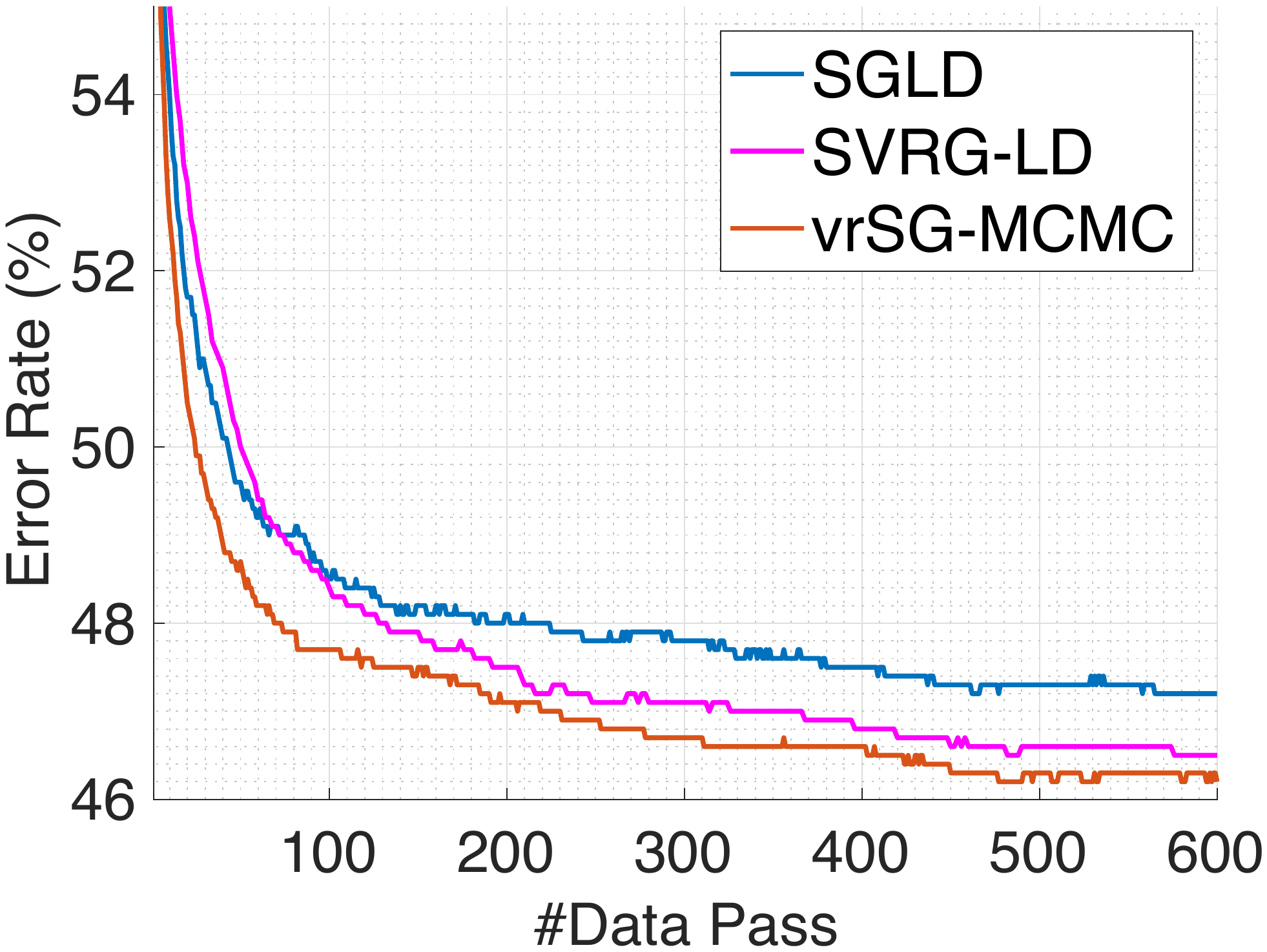}
		\end{minipage}
		\begin{minipage}{0.49\linewidth}
			\includegraphics[width=\columnwidth]{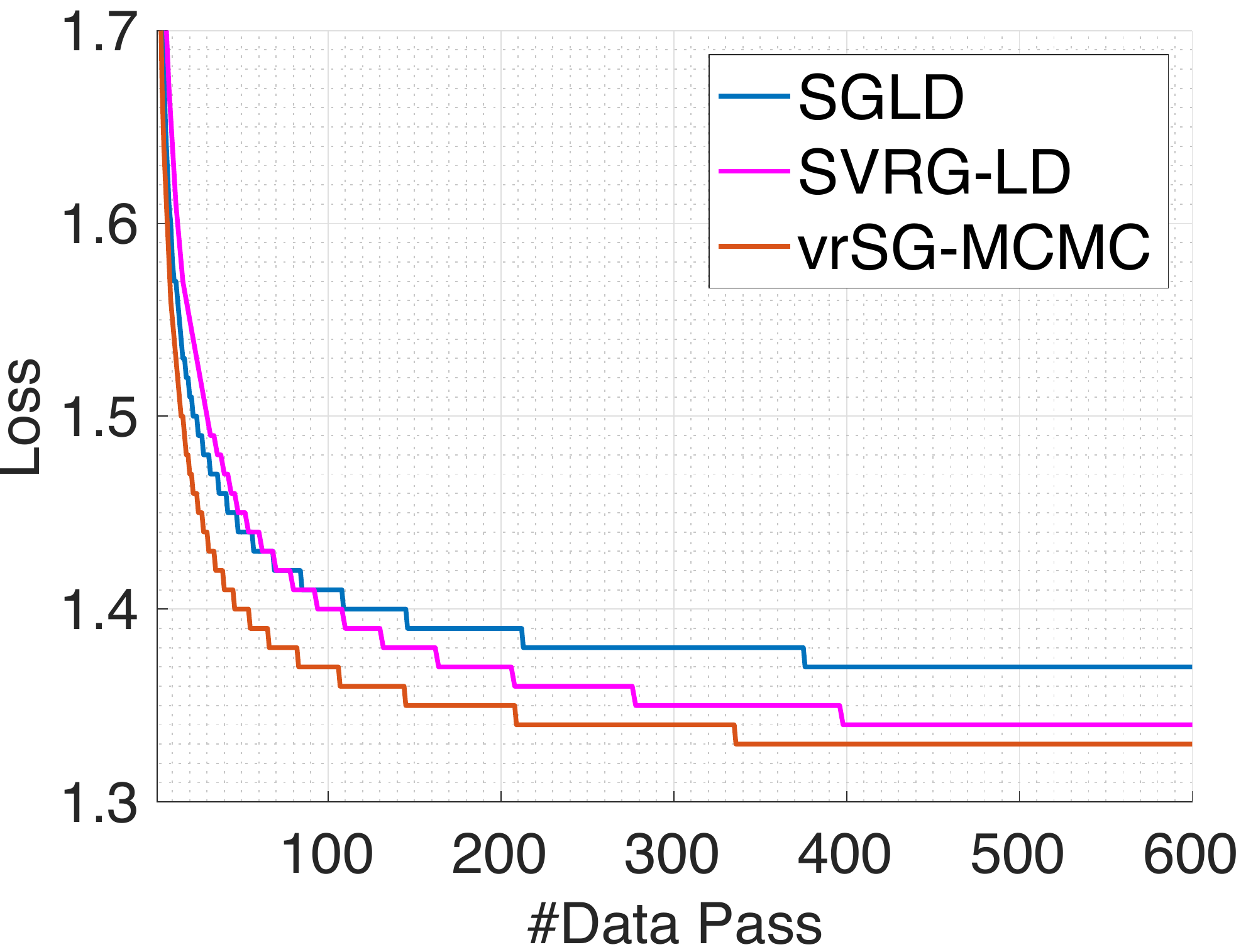}
		\end{minipage}
		\caption{Number of passes through data vs. testing error (left) / loss (right) on MNIST (top) and CIFAR-10 (bottom) datasets.}
		\label{fig:fnn}
	\end{center}
	\vskip -0.2in
\end{figure} 

\subsection{Convolutional neural networks}
We use the CIFAR-10 dataset, and test two CNN architectures for image classification. The first architecture is a deep convolutional neural networks with 4 convolutional layers, denoted as C32-C32-C64-C32, where max-pooling is applied on the output of the first three convolutional layers, and a Dropout layer is applied on the output of the last convolutional layer. The second architecture is a 20-layers deep residual network (ResNet) with the same setup as in \cite{he2016deep}. Specifically, we use a step-size-decrease scheme as $h_l = \frac{1}{10+\text{1.8e-3}\times l}$ for both vrSG-MCMC and SGLD, where $l$ is the number of iterations so far.

Figure~\ref{fig:cnn} plots the number of passes through the data versus test error/loss on both models. Similar to the results on MLP, vrSG-MCMC converges much faster than SGLD, leading to lower test errors and loss. Interestingly, the gap seems larger in the more complicated ResNet architecture; furthermore, the learning curves look much less noisy (smoother) for vrSG-MCMC because of the reduced variance in stochastic gradients.

\begin{figure}[ht]
	\begin{center}
		\begin{minipage}{0.49\linewidth}
			\includegraphics[width=\columnwidth]{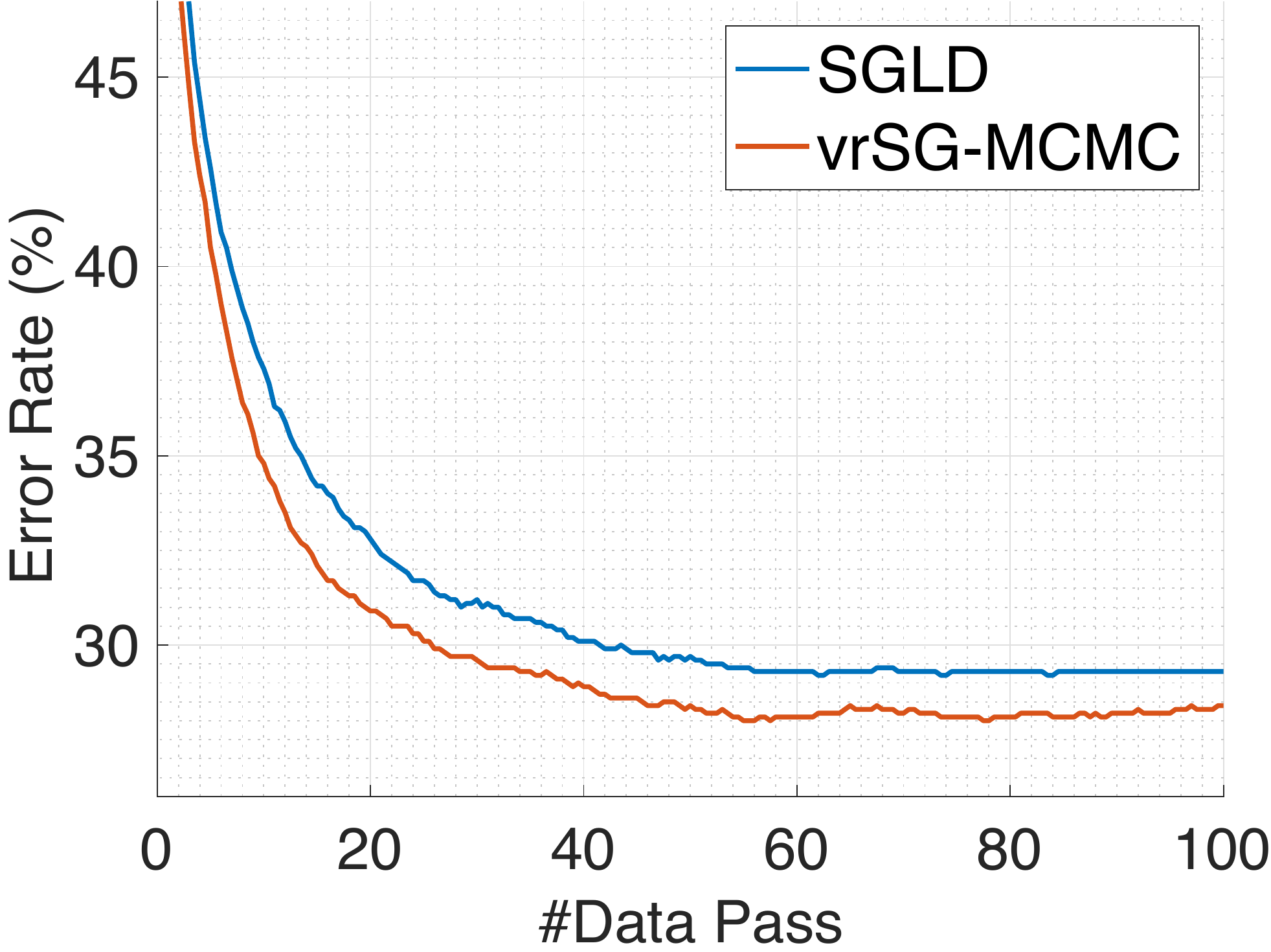}
		\end{minipage}
		\begin{minipage}{0.49\linewidth}
			\includegraphics[width=\columnwidth]{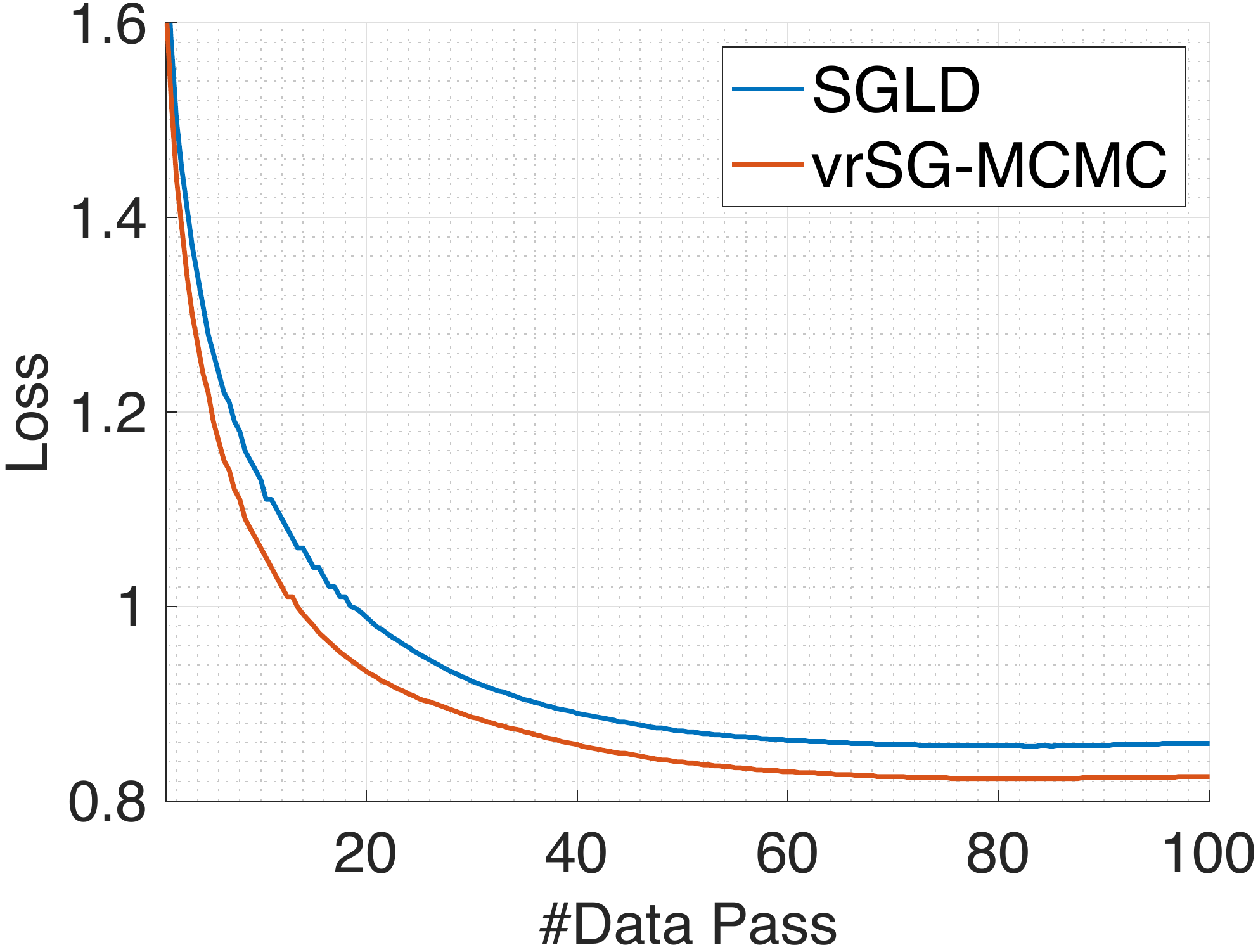}
		\end{minipage}
		\begin{minipage}{0.49\linewidth}
			\includegraphics[width=\columnwidth]{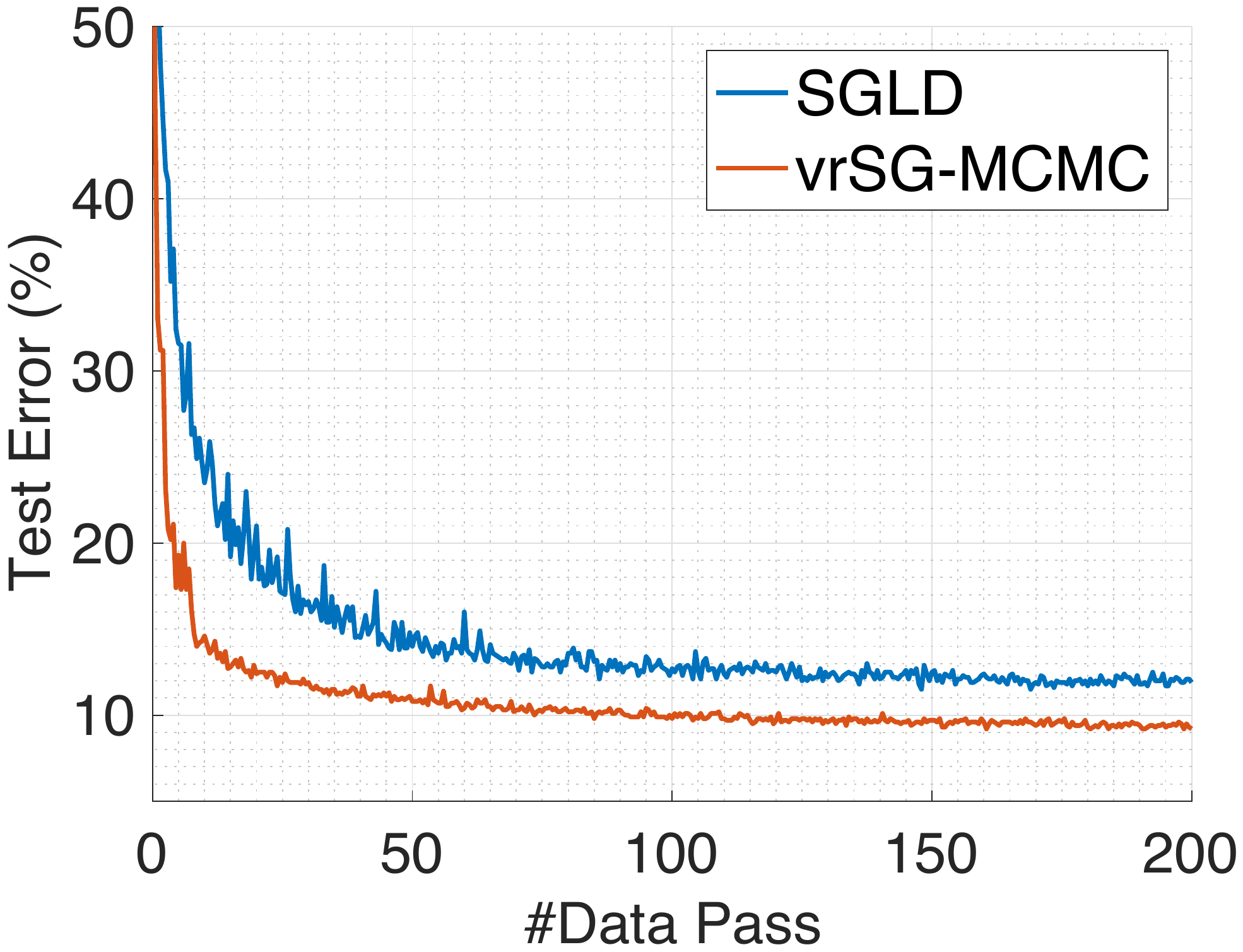}
		\end{minipage}
		\begin{minipage}{0.49\linewidth}
			\includegraphics[width=\columnwidth]{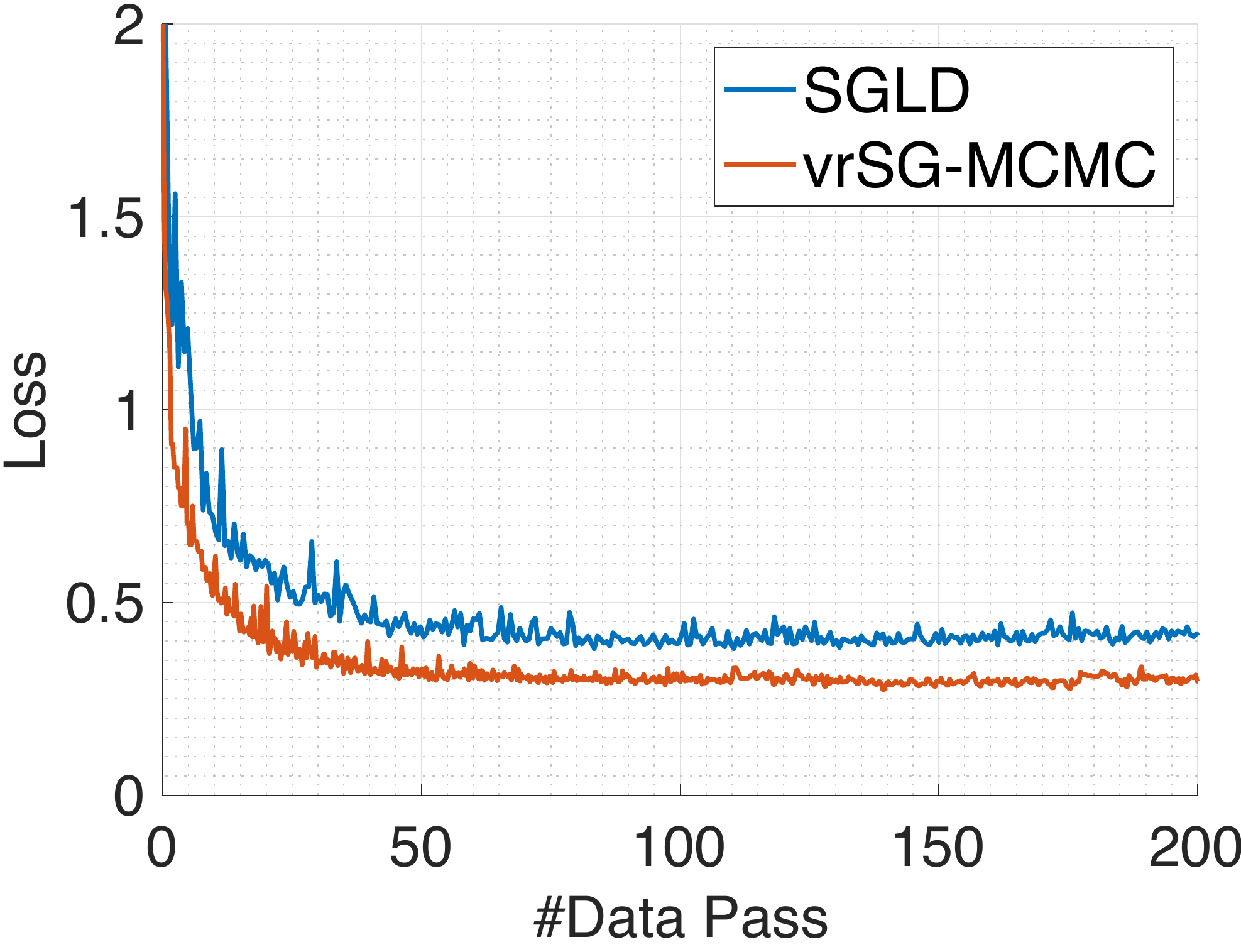}
		\end{minipage}
		\caption{Number of passes through data vs. testing error (left) / loss (right) with CNN-4 (top) and ResNet (bottom) on CIFAR-10.}
		\label{fig:cnn}
	\end{center}
	\vskip -0.2in
\end{figure}

\subsection{Recurrent neural networks}

The recurrent neural network with LSTM units \cite{Hochreiter:1997:LSM} is a powerful architecture used for modeling sequence-to-sequence data. 
We consider the task of language modeling on two datasets, \emph{i.e.}, the Penn Treebank (PTB) dataset and WikiText-2 dataset \cite{merity2016pointer}. PTB is the smaller dataset among the two, containing a vocabulary of size 10,000. We use the default setup of 887,521 tokens for training, 70,390 for validation and 78,669 for testing. WikiTest-2 is a large dataset with 2,088,628 tokens from 600 Wiki articles for training, 217,649 tokens from 60 Wiki articles for validation, and 245,569 tokens from an additional 60 Wiki articles for testing. The total vocabulary size is 33,278.

We adopt the hierarchical LSTM achitecture \cite{zaremba2014recurrent}. The hierarchy depth is set to 2, with each LSTM containing 200 hidden unites. The step size is set to 0.5 for both datasets. For more stable training, standard gradient clipping is adopted, where gradients are clipped if the norm of the parameter vector exceeds 5. Figure~\ref{fig:rnn} plots the number of passes through the data versus test perplexity on both datasets. The results are consistent with the previous experiments on MLPs and CNNs, where vrSG-MCMC achieves faster convergence than SGLD; its learning curves in terms of testing error/loss are also much smoother.

\begin{figure}[ht]
	\begin{center}
		\begin{minipage}{0.49\linewidth}
			\includegraphics[width=\columnwidth]{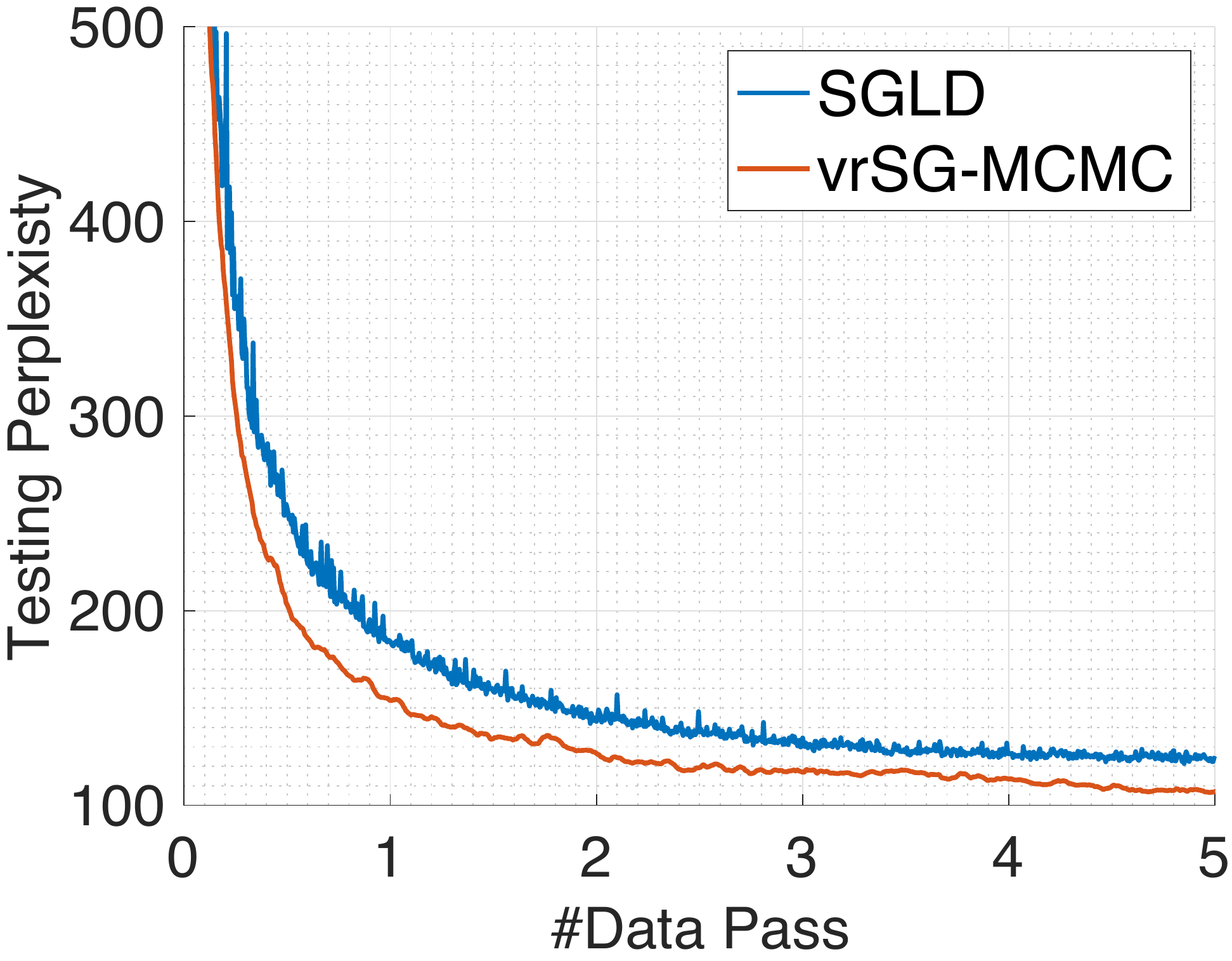}
		\end{minipage}
		\begin{minipage}{0.49\linewidth}
			\includegraphics[width=\columnwidth]{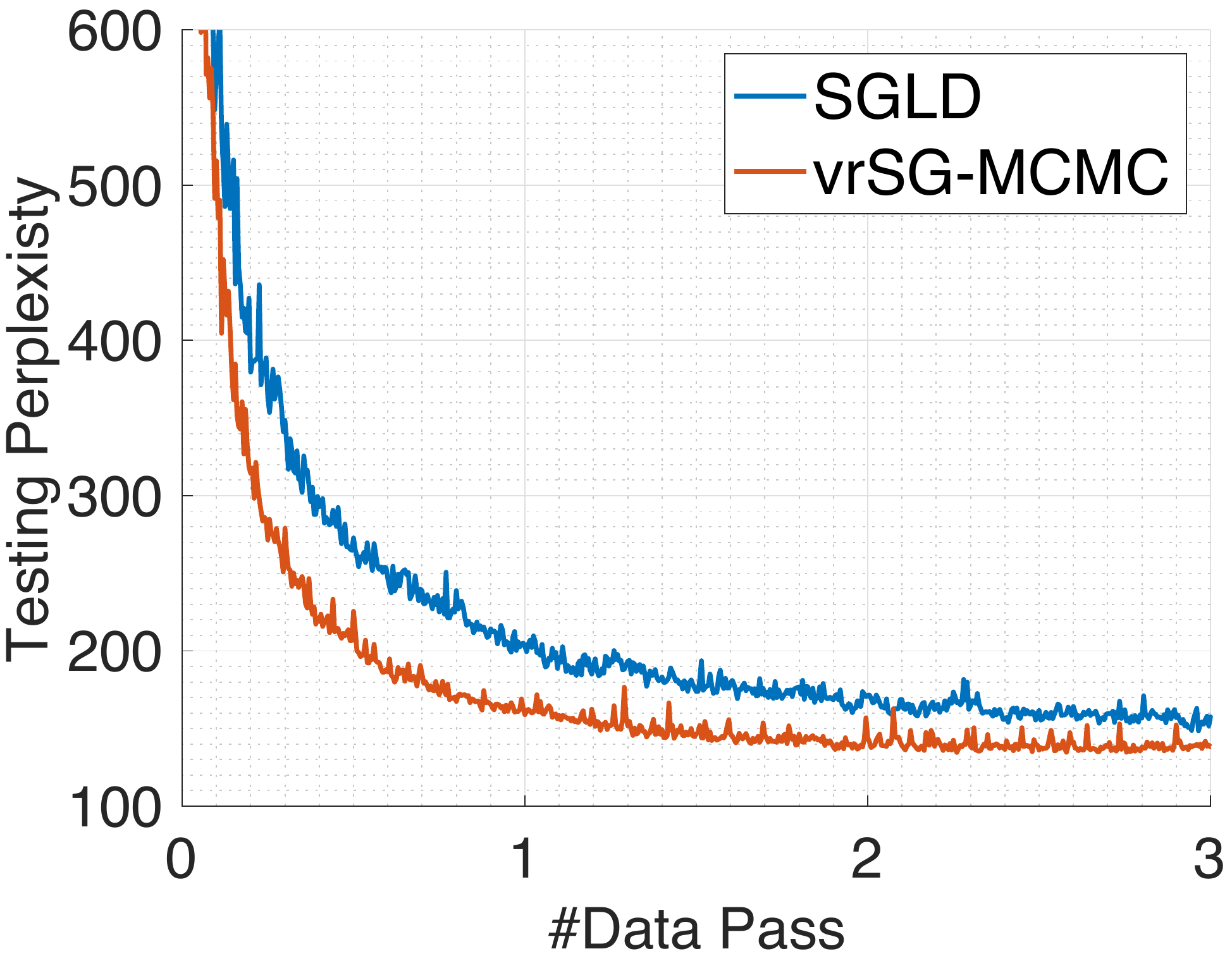}
		\end{minipage}
		\caption{Number of passes through data vs. testing perplexity on the PTB dataset (top) and WikiTest-2 dataset (bottom).}
		\label{fig:rnn}
	\end{center}
	\vskip -0.2in
\end{figure}

\subsection{Parameter sensitivity}\label{sec:sensity}

Note that one of the main differences between vrSG-MCMC and the recently proposed SVRG-LD \cite{DubeyRPSX:nips16} is that the former uses minibatches of size $n_1$ to calculate the {\em old gradient} $\tilde{g}$ in Algorithm~\ref{alg:pvrsgmcmc}, leading to a much more computationally efficient algorithm, with theoretical guarantees. This section tests the sensitivity of model performance to the parameter $n_1$. 

\begin{figure}[ht]
	\vskip -0.1in
	\begin{center}
		\begin{minipage}{0.75\linewidth}
			\includegraphics[width=\columnwidth]{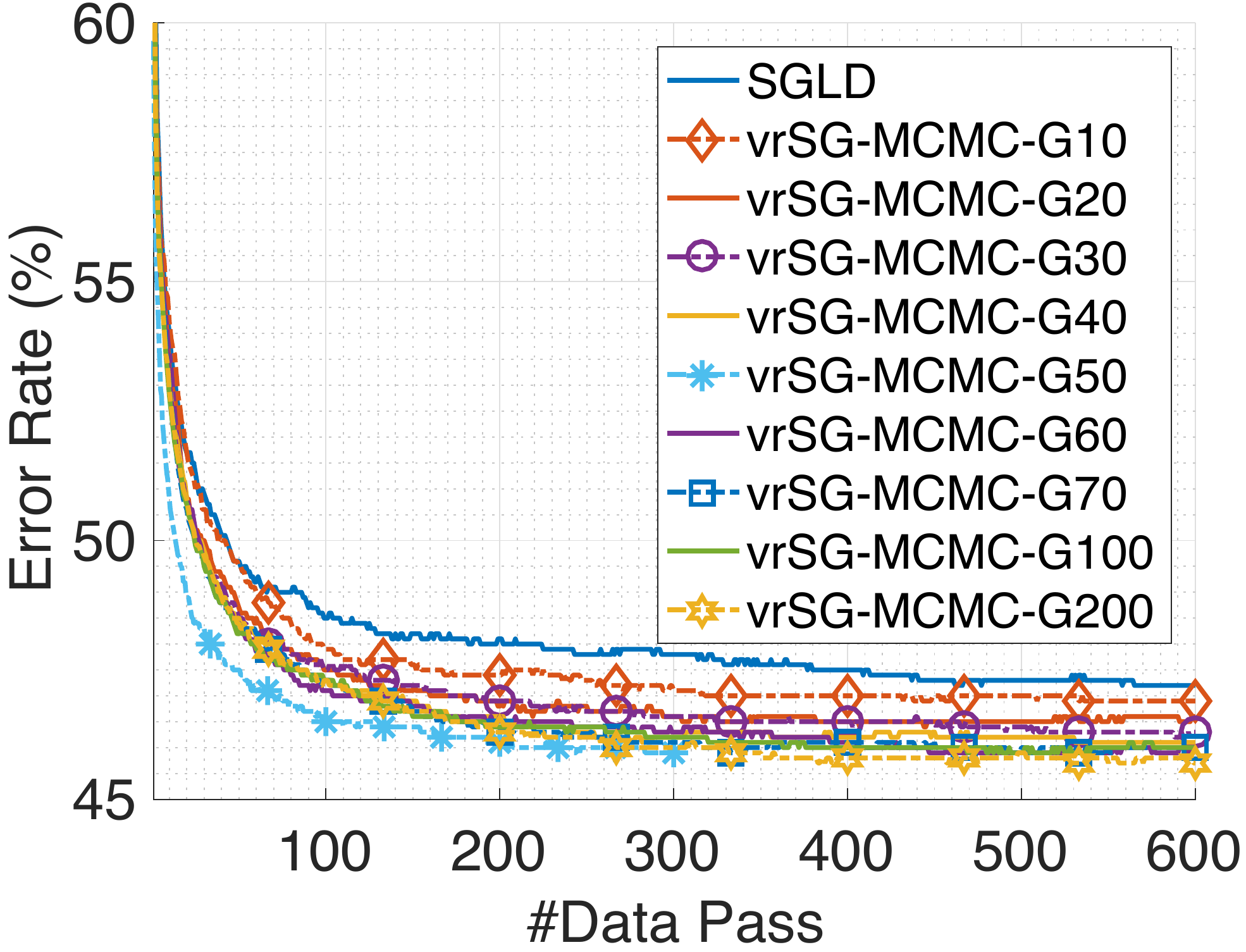}
		\end{minipage}
		\begin{minipage}{0.75\linewidth}
			\includegraphics[width=\columnwidth]{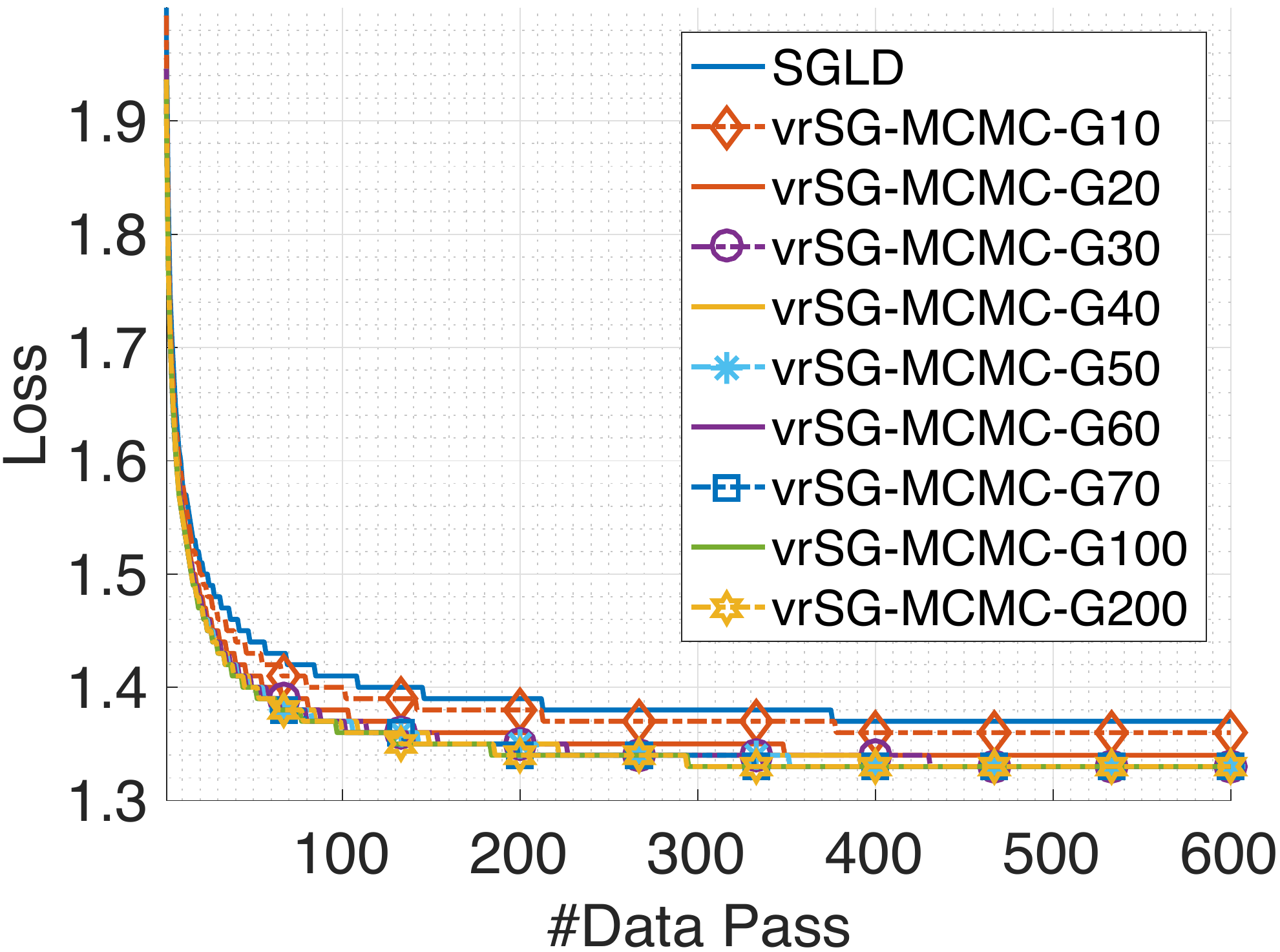}
		\end{minipage}\vskip -0.1in
		\caption{Number of passes through data vs. testing errors (top) / loss (bottom) on the CIFAR-10 dataset, with varying $n_1$ values.}
		\label{fig:n1}
	\end{center}
	\vskip -0.2in
\end{figure}

For simplicity, we run on the same MLP model described in Section~\ref{sec:exp_mlp} on the CIFAR-10 dataset, where the same parameter settings are used, but varying $n_1$ in $\{100, 200, 300, 400, 500, 600, 700, 1000, 2000\}$.
Figure~\ref{fig:n1} plots the number of passes through data versus test errors/loss, where we use ``vrSG-MCMC-Gn'' to denote vrSG-MCMC with $n_1 = n$. Interestingly, vrSG-MCMC outperforms the baseline SGLD on all $n_1$ values. Notably, when $n_1$ is large enough ($n_1 = 200$ in our case), their corresponding test errors and loss are very close. This agrees with the intuition that computing the {\em old gradient} using the whole training data is not necessarily a good choice in order to balance the stochastic gradient noise and computational time.

\section{Conclusion}
We investigate the impact of minibatches in SG-MCMC and propose a practical variance-reduction SG-MCMC algorithm to reduce the stochastic gradient noise in SG-MCMC. Compared to existing variance reduction techniques for SG-MCMC, the proposed method is efficient from both computational and storage perspectives. Theory is developed to guarantee faster convergence rates of vrSG-MCMC compared to standard SG-MCMC algorithms. Extensive experiments on Bayesian learning of deep neural networks verify the theory, obtaining significant speedup compared to the corresponding SG-MCMC algorithms.

\newpage
{\small
	\bibliographystyle{aaai}
	\bibliography{references}
}

\newpage
\appendix

\twocolumn[
\begin{center}
	\bf{\Large Supplementary Materials for:\\ A Theory for A Class of Practical Variance-Reduction Stochastic Gradient MCMC}
\end{center}\vspace{1cm}
]

\section{Basic Setup for Stochastic Gradient MCMC}\label{app:ass}

Given data $\Db = \{\db_1, \cdots, \db_N\}$, a generative model
$$p(\Db | \thetab) = \prod_{i=1}^N p(\db_i | \thetab)~,$$
with model parameter $\thetab \in \mathbb{R}^r$, and prior $p(\thetab)$, we want to compute the posterior distribution:
$$\rho(\thetab) \triangleq p(\thetab | \Db) \propto p(\Db | \thetab) p(\thetab) \triangleq e^{-U(\thetab)}~,$$
where
\begin{align}\label{eq:grad}
\nabla_{\thetab} U(\thetab) = -\nabla_{\thetab} \log p(\thetab) - \sum_{i=1}^N \nabla_{\thetab} \log p(\db_i | \thetab)~.
\end{align}

Consider the SDE:
\begin{align}\label{eq:ld}
\mathrm{d}\xb_t &= F(\xb_t)\mathrm{d}t + g(\xb_t)\mathrm{d}\mathcal{\wb}_t~,
\end{align}
where $\xb \in \mathbb{R}^d$ is the state variable, typically $\xb \supseteq \thetab$ is an augmentation of the model parameter, thus  $r \le d$; $t$ is the time index, $\mathcal{\wb}_t \in \mathbb{R}^d$ is 
$d$-dimensional Brownian motion; functions $F: \mathbb{R}^r \to \mathbb{R}^d$ and $g: \mathbb{R}^d \rightarrow \mathbb{R}^{d\times d}$ are assumed to satisfy the usual Lipschitz continuity condition \cite{Ghosh:book11}. In Langevin dynamics, we have $\xb = \thetab$ and
\begin{align*}
F(\thetab_t) &= -\nabla_{\thetab} U(\thetab_t) \\
g(\thetab_t) &= \sqrt{2}~.
\end{align*}

For the SDE in \eqref{eq:ld}, the generator $\mathcal{L}$ is defined as:
\begin{align}
\mathcal{L} \psi \triangleq \frac{1}{2}\nabla\psi \cdot F + \frac{1}{2}g(\thetab) g(\thetab)^{*}:D^2 \psi~,
\end{align}
where $\psi$ is a measurable function, $D^k \psi$ means the $k$-derivative of $\psi$, $*$ means transpose. $\ab\cdot\bb \triangleq \ab^T\bb$ for two vectors $\ab$ and $\bb$, $\Ab:\Bb \triangleq \mbox{trace}(\Ab^T\Bb)$ for two matrices $\Ab$ and $\Bb$. Under certain
assumptions, we have that there exists a function $\phi$ on $\mathbb{R}^d$ such that the following Poisson
equation is satisfied \cite{MattinglyST:JNA10}:
\begin{align}\label{eq:poissoneq}
\mathcal{L} \psi = \phi - \bar{\phi}~,
\end{align}
where $\bar{\phi} \triangleq \int \phi(\thetab) \rho(\mathrm{d} \thetab)$ denotes the model average, with $\rho$ being the equilibrium distribution for
the SDE \eqref{eq:ld}.

In stochastic gradient Langevin dynamics (SGLD), we update
the parameter $\thetab$ at step $l$, denoted as $\thetab_l$\footnote{Strictly speaking, $\thetab$ should be indexed by ``time'' instead of ``step'', {\it i.e.}, $\thetab_{\sum_{l = 1}^\prime=1^l h_{l^\prime}}$ instead of $\thetab_l$. We adopt the later for notation simplicity in the following. This applies for the general case of $\xb$.}, using the following descreatized method:
\begin{align*}
\thetab_{l+1} = \thetab_{l} -\nabla_{\thetab}\tilde{U}_l(\thetab_l) h_{l+1} + \sqrt{2h_{l+1}} \zetab_{l+1}~,
\end{align*}
where $h_{l+1}$ is the step size, $\zetab_{l}$ a Gaussian random variable with mean 0 and variance 1,
$\nabla_{\thetab}\tilde{U}_l$ is an unbiased estimate of $\nabla_{\thetab}U$ in \eqref{eq:grad} with 
a random minibatch of size n, {\it e.g.},
\begin{align}
\nabla_{\thetab} \tilde{U}_l(\thetab_l) = \nabla_{\thetab} \log p(\thetab_l) + \frac{N}{n}\sum_{i=1}^n \nabla_{\thetab} \log p(\xb_{\pi_i} | \thetab_l)~,
\end{align}
where $\{\pi_1, \cdots, \pi_n\}$ is a subset of a random permutation of $\{1, \cdots, N\}$.

In our analysis, we are interested in the {\em mean square error} (MSE) at iteration $L$, defined as
\begin{align*}
\mbox{MSE}_L \triangleq \mathbb{E}\left(\hat{\phi}_L - \bar{\phi}\right)^2~,
\end{align*}
where $\hat{\phi}_L \triangleq \frac{1}{L} \sum_{l=1}^L \phi(\thetab_l)$ denotes the sample average, $\bar{\phi}$ is the true posterior average
defined in  \eqref{eq:poissoneq}.

In this paper, for the function $f: \mathbb{R}^m \rightarrow \mathbb{R}$ in an $\mathcal{L}^p$ space, {\it i.e.}, a space of functions for which the $p$-th power of the absolute value is Lebesgue integrable, we consider the standard norm $\|f\|_p$ defined as ($\|f\|_{\infty}$ is simplified as $\|f\|$):
\begin{align*}
\|f\|_p \triangleq \left(\int_{\mathbb{R}^m} |f(\xb)|^p\mathrm{d}\xb\right)^{1/p} < \infty~.
\end{align*}
In order to guarantee well-behaved SDEs and the corresponding numerical integrators, following existing literatures such as \cite{VollmerZT:arxiv15,ChenDC:NIPS15}, we impose the following assumptions.

\begin{assumption}\label{ass:assumption1}
	The SDE \eqref{eq:ld} is ergodic. Furthermore, the solution of \eqref{eq:poissoneq} exists, and the solution functional $\psi$ of the Poisson equation \eqref{eq:poissoneq} satisfies the following properties:
	\begin{itemize}
		\item $\psi$ and its up to 3th-order derivatives $\mathcal{D}^k \psi$, are bounded by a function $\mathcal{V}$, {\it i.e.}, $\|\mathcal{D}^k \psi\| \leq C_k\mathcal{V}^{p_k}$ for $k=(0, 1, 2, 3, 4)$, $C_k, p_k > 0$.
		\item The expectation of $\mathcal{V}$ on $\{\xb_{l}\}$ is bounded: $\sup_l \mathbb{E}\mathcal{V}^p(\xb_{l}) < \infty$.
		\item $\mathcal{V}$ is smooth such that $\sup_{s \in (0, 1)} \mathcal{V}^p\left(s\xb + \left(1-s\right)\yb\right) \leq C\left(\mathcal{V}^p\left(\xb\right) + \mathcal{V}^p\left(\yb\right)\right)$, $\forall \xb \in \mathbb{R}^m, \yb \in \mathbb{R}^m, p \leq \max\{2p_k\}$ for some $C > 0$.
	\end{itemize}
\end{assumption}

\section{Proofs of Extended Results for Standard SG-MCMC}

First, according to the definition of $\Delta V_l$, we note that $\Delta V_l\psi = (\nabla_{\thetab}U_l(\thetab) - \nabla_{\thetab}\tilde{U}_l(\thetab))\!\cdot\!\nabla \psi$ for the solution functional $\psi$ of the Poisson equation~\ref{eq:poissoneq}. Since $\|\Delta V_l \psi\| \leq \|\nabla_{\thetab}U_l(\thetab) - \nabla_{\thetab}\tilde{U}_l(\thetab)\|\|\nabla \psi\|$, and $\|\nabla\psi\|$ is assumed to be bounded for a test function $\psi$, we omit the operator $\nabla$ in our following analysis (which only contributes to a constant), manifesting a slight abuse of notation for conciseness.

The proofs of Lemma~\ref{lem:exp_alpha} and Theorem~\ref{theo:mse_batch} are closely related. We will first prove Theorem~\ref{theo:mse_batch}, the proof for Lemma~\ref{lem:exp_alpha} is then directly followed.

\begin{proof}[Proof of Theorem~\ref{theo:mse_batch}]
	
	Let $\alphab_{li} = \nabla_{\thetab}\log p(\db_i | \thetab_{l})$, and
	\begin{align*}
	z_i = \left\{\begin{array}{ll}
	1 & \text{ if data $i$ is selected} \\
	0 & \text{ otherwise}
	\end{array}\right.~,
	\end{align*}
	then we have
	\begin{align*}
	\Delta V_l &= \sum_{i=1}^N \mathbb{E}\alphab_{li} \left(1 - \frac{N}{n}z_i\right) \\
	\rightarrow \mathbb{E}\left|\Delta V_l\right|^2 &= \sum_{i=1}^N \sum_{j=1}^N \mathbb{E} \alphab_{li} \mathbb{E} \alphab_{lj} \left(1 - \frac{N}{n}z_i\right)\left(1 - \frac{N}{n}z_j\right)~.
	\end{align*}
	
	Since
	\begin{align*}
	\mathbb{E} z_i =& \frac{1}{N} + \frac{N-1}{N} \frac{1}{N-1} + \cdots \\
	+& \frac{N-1}{N} \frac{N-2}{N-1}\cdots \frac{N-m+1}{N-m+2} \frac{1}{N-m+1} \\
	=& \frac{n}{N}~,
	\end{align*}
	we have $\mathbb{E}\Delta V_l = 0$, {\it i.e.}, $\nabla \tilde{U}_l(\theta)$ is an unbiased estimate of $\nabla U(\theta)$.
	
	In addition, we have
	\begin{align*}
	&\mathbb{E}\left(1 - \frac{N}{n}z_i\right)\left(1 - \frac{N}{n}z_j\right) \\
	&= \mathbb{E}\left[1 - \frac{N}{n}z_i - \frac{N}{n}z_j + \frac{N^2}{n^2}z_iz_j\right] \\
	&= 1 - 2\frac{N}{n}\frac{n}{N} + \frac{N^2}{n^2}\mathbb{E}z_iz_j \\
	&= \frac{N^2}{n^2}\mathbb{E}z_iz_j - 1~.
	\end{align*}
	
	When $i = j$,
	\begin{align*}
	&\mathbb{E}\left(1 - \frac{N}{n}z_i\right)\left(1 - \frac{N}{n}z_j\right) = \frac{N^2}{n^2}\mathbb{E}z_i^2 - 1 \\
	=&\frac{N^2}{n^2}\mathbb{E}z_i - 1 = \frac{N}{n} - 1~.
	\end{align*}
	
	When $i \neq j$, because
	\begin{align*}
	\mathbb{E}z_iz_j &= p(i \mbox{ selected}) p(j\mbox{ selected} | i\mbox{ selected}) \\
	&= \frac{n}{N} \frac{n-1}{N-1}~.
	\end{align*}
	We have
	\begin{align*}
	&\mathbb{E}\left(1 - \frac{N}{n}z_i\right)\left(1 - \frac{N}{n}z_j\right) = \frac{N^2}{n^2}\mathbb{E}z_i z_j - 1 \\
	=&\frac{N}{n}\frac{n-1}{N-1} - 1~.
	\end{align*}
	
	As a result,
	\begin{align}\label{eq:deltavl}
	&\mathbb{E}\left|\Delta V_l\right|^2 \nonumber\\
	=& \left(\sum_{i=1}^N \mathbb{E}\alphab_{li}^2\right)\left(\frac{N}{n} - 1\right) + 2 \sum_{i < j}\mathbb{E}\alphab_{li}\alphab_{lj} \left(\frac{N}{n} \frac{n-1}{N-1} - 1\right) \nonumber\\
	=& \left(\frac{N}{n} - 1\right) \sum_{i,j}^N \mathbb{E}\alphab_{li}\alphab_{lj} + 2\sum_{i < j}\mathbb{E}\alphab_{li} \alphab_{lj} \left(\frac{N}{n} \frac{n-1}{N-1} - \frac{N}{n}\right) \nonumber\\
	=& \left(\frac{N}{n} - 1\right) \sum_{i,j}^N \mathbb{E}\alphab_{li}\alphab_{lj} - 2\sum_{i < j}\mathbb{E}\alphab_{li}\alphab_{lj} \frac{N}{n} \frac{N-n}{N-1} \nonumber\\
	=&\frac{(N - n)N^2}{n}\left(\frac{1}{N^2}\sum_{i,j}\mathbb{E}\alphab_{li}\alphab_{lj} - \frac{2}{N(N-1)}\sum_{i\leq j}\mathbb{E}\alphab_{li}\alphab_{lj}\right) \nonumber\\
	\triangleq& \frac{(N - n)N^2}{n} \Gamma_l~.
	\end{align}
	
	Because we assume using a 1st-order numerical integrator, according to Lemma~\ref{lem:biasmse}, and combining \eqref{eq:deltavl} from above, 
	we have the bound for the MSE $\mathbb{E}\left(\hat{\phi}_L - \bar{\phi}\right)^2$ as:
	\begin{align*}
	\mathbb{E}\left(\hat{\phi}_L - \bar{\phi}\right)^2 
	\leq C\left(\frac{(N-n)N^2\Gamma_M}{nL} + \frac{1}{Lh} + h^2\right)~.
	\end{align*}
	
\end{proof}

\begin{proof}[Proof of Lemma~\ref{lem:exp_alpha}]
	The lemma follows directly from \eqref{eq:deltavl} and the fact that $$\mathbb{E}\left|\Delta V_l\right|^2 \geq 0~.$$
\end{proof}

\begin{proof}[Proof of the optimal MSE bound of Theorem~\ref{theo:mse_batch}]
	
	From the assumption, we have
	\begin{align}\label{eq:runtime}
	T \propto nL~.
	\end{align}
	The MSE bounded is obtained by directly substituting \eqref{eq:runtime} into the MSE bound in Lemma~\ref{lem:exp_alpha}, resulting in
	\begin{align*}
	\mbox{MSE: }\mathbb{E}\left(\hat{\phi}_L - \bar{\phi}\right)^2 
	&\leq C \left(\frac{(N-n)N^2 \Gamma_M}{T} + \frac{n}{T h} + h^{2}\right)
	\end{align*}
	After optimizing the above bound over $h$, we have
	\begin{align}\label{eq:msebound_t}
	\mathbb{E}\left(\hat{\phi}_L - \bar{\phi}\right)^2 \leq C \left(\frac{(N-n)N^2 \Gamma_M}{T} + \frac{n^{2/3}}{T^{2/3}}\right)~.
	\end{align}
	
\end{proof}

\begin{proof}[Proof of Corollary~\ref{cor:optimalT}]
	
	To examine the property of the MSE bound \eqref{eq:msebound_t} w.r.t.\! $n$, we first note that the derivative can be written as:
	\begin{align*}
	f \triangleq \frac{\partial}{\partial n}\mathbb{E}\left(\hat{\phi}_L - \bar{\phi}\right)^2 = O\left(\frac{2}{3T^{2/3}n^{1/3}} - \frac{\Gamma_M N^2}{T}\right)~.
	\end{align*}
	As a result, we have the following three cases:
	\begin{itemize}
		\item[1)] When $f < 0$, {\it i.e.}, the bound is decreasing when $n$ increasing, we have $T < \frac{27}{8}\Gamma_M^3N^6n$. Because $n$ is in the range of $[1, N]$, and we require $f < 0$ for all $n$'s, the minimum value of $\frac{27}{8}\Gamma_M^3N^6n$ is obtained when taking $n = 1$. Consequently, we have that when $T < \frac{27}{8}\Gamma_M^3N^6$, the optimal MSE bound \eqref{eq:msebound_t} is decreasing w.r.t.\! $n$. The minimum MSE bound is thus achieved at $n = N$. This case corresponds to the limited-computation-budget case.
		\item[2)] When $f > 0$, {\it i.e.}, the bound is increasing when $n$ increasing, we have $T > \frac{27}{8}\Gamma_M^3N^6n$. Because $n$ is in the range of $[1, N]$, and we require $f > 0$ for all $n$'s, the maximum value of $\frac{27}{8}\Gamma_M^3N^6n$ is obtained when taking $n = N$. Consequently, we have that when $T > \frac{27}{8}\Gamma_M^3N^7$, the optimal MSE bound \eqref{eq:msebound_t} is increasing w.r.t.\! $n$. The minimum MSE bound is thus achieved at $n = 1$. This case corresponds to the long-run case (computational budget is large enough).
		\item[3)] When the computational budget is in between the above two cases, the optimal MSE bound \eqref{eq:msebound_t} first increases then decreases w.r.t.\! $n$ in range $[1, N]$. The optimal MSE bound is thus obtained either at $n = 1$ or at $n = N$, depending on $(N, T, \Gamma_M)$.
	\end{itemize}
	
\end{proof}

\section{Proofs of theorems for vrSG-MCMC}\label{app:theory}


\begin{proof}[Proof of Lemma~\ref{lem:alpha_beta}]
	From the definitions, we know that $\alphab_{li}$ is the same as $\betab_{li}$ except evaluating on different model parameters, denoted as $\thetab_{l}$ and $\tilde{\thetab}_l$, respectively. Note that $\tilde{\thetab}_l$ is an {\em outdated} version of $\thetab_l$, with difference at most $m$. The proof of Lemma~\ref{lem:alpha_beta} is then an application of a lemma from \cite{ChenDLZC:NIPS16}, which is stated in Lemma below.
	\begin{lemma}[Lemma~8 in \cite{ChenDLZC:NIPS16}]\label{lem:old}
		Let $\thetab_{l}$ and $\tilde{\thetab_{l}}$ be two parameters where $\tilde{\thetab_{l}}$ is $\tau$-step older than $\thetab_{l}$, then we have
		\begin{align*}
		\left\|\mathbb{E}\left(\nabla_{\thetab}\log p(\db | \thetab_{l}) - \nabla_{\thetab}\log p(\db | \tilde{\thetab}_{l})\right)\right\| = O(\tau h)~.
		\end{align*}
	\end{lemma}
	Based on the definitions in Algorithm~\ref{alg:pvrsgmcmc}, we can consider $\betab_{li}$ as an outdated version of $\alphab_{li}$, with time difference $m$. As a result, Lemma~\ref{lem:alpha_beta} follows by replacing $\tau$ with $m$ in Lemma~\ref{lem:old}.
\end{proof}

The following is a formal proof of the unbiasness of $\nabla_{\thetab}\tilde{U}(\thetab_{l})$, stated in the ``Convergence rate'' section in the main text.

\begin{proof}[Proof of the unbiasness of $\nabla_{\thetab}\tilde{U}(\thetab_{l}) $]
	
	First note that in variance reduction, the following stochastic gradient is used:
	\begin{align}
	\nabla_{\thetab}\tilde{U}(\thetab_{l}) =& \frac{N}{n_2}\sum_{i=1}^N\left(\nabla_{\thetab}\log p(\xb_{i}|\thetab_{l}) - \nabla_{\thetab}\log p(\xb_{i}|\tilde{\thetab}_{l})\right)z_i \nonumber\\
	+& \frac{N}{n_1}\sum_{i=1}^N\sum_{i=1}^N\nabla_{\thetab}\log p(\xb_{i}|\tilde{\thetab}_{l}) b_i~.
	\end{align}
	
	\begin{align}\label{eq:deltaVl}
	\Delta V_l = \sum_{i=1}^N\alphab_{li}\left(1 - \frac{N}{n_2}z_i\right) + \sum_{i=1}^N\betab_{li}\left(\frac{N}{n_2}z_i - \frac{N}{n_1}b_i\right)~.
	\end{align}
	
	Because $\mathbb{E}z_i = \frac{n_2}{N}$, $\mathbb{E}b_i = \frac{n_1}{N}$,
	it is easy to verify that $\mathbb{E}\Delta V_l = 0$. As a result, the unbiasness holds.
	
\end{proof}

\section{Proof of Theorem~\ref{theo:main}}

Before proving Theorem~\ref{theo:main}, let us first simplify $\mathbb{E}\|\Delta V_l\|^2$ in the MSE bound. In the following, we decompose it into several terms which can be simplified separately. Our goal is to show that the proposed vrSG-MCMC algorithm induces a smaller $\mathbb{E}\|\Delta V_l\|^2$ term, thus leading to a faster convergence rate.
Note we can rewrite $\Delta V_l$ in terms of $\{\alphab_{li}, \betab_{li}, z_i, b_i\}$ as:
\begin{align*}
\Delta V_l = \sum_{i=1}^N\alphab_{li}\left(1 - \frac{N}{n_2}z_i\right) + \sum_{i=1}^N\betab_{li}\left(\frac{N}{n_2}z_i - \frac{N}{n_1}b_i\right)~.
\end{align*}

Consequently, we have
\begin{align}\label{eq:expDelVl}
\mathbb{E}&\|\Delta V_l\|^2 = \underbrace{\sum_{i,j}\mathbb{E}\alphab_{li}^T\alphab_{lj}\left(1 - \frac{N}{n_2}z_i\right)\left(1 - \frac{N}{n_2}z_j\right)}_{A_l} \nonumber \\
&+ \underbrace{\sum_{i,j}\mathbb{E}\betab_{li}^T\betab_{lj}\left(\frac{N}{n_2}z_i - \frac{N}{n_1}b_i\right)\left(\frac{N}{n_2}z_j - \frac{N}{n_1}b_j\right)}_{B_l} \nonumber \\
&+ \underbrace{2\sum_{i,j}\mathbb{E}\alphab_{li}^T\betab_{lj}\left(1 - \frac{N}{n_2}z_i\right)\left(\frac{N}{n_2}z_j - \frac{N}{n_1}b_j\right)}_{C_l}~.
\end{align}

Now \eqref{eq:expDelVl} can be further simplified by summing over all the binary random variables $\{z_i\}$ and $\{b_i\}$. After summing out the binary random variables $\{z_i, b_i\}$, we arrive formula summarized in the following proposition:

\begin{proposition}\label{prop:MSE_vr}
	The terms $A_l$, $B_l$ and $C_l$ in \eqref{eq:expDelVl} can be simplified as:
	\begin{align*}
	A_l &= \left(\frac{N}{n_2} - 1\right)\sum_{ij}\mathbb{E}\alphab_{li}^T\alphab_{lj} - 2\frac{N(N-n_2)}{n_2(N-1)}\sum_{i<j}\mathbb{E}\alphab_{li}^T\alphab_{lj} \\
	B_l &= \left(\frac{N}{n_2} + \frac{N}{n_1} - 2\right)\sum_{ij}\mathbb{E}\betab_{li}^T\betab_{lj} \\
	&~~- 2\left(\frac{N(N-n_2)}{n_2(N-1)} + \frac{N(N-n_1)}{n_1(N-1)}\right)\sum_{i<j}\mathbb{E}\betab_{li}^T\betab_{lj} \\
	C_l &= 2\left(1 - \frac{N}{n_2}\right)\sum_{ij}\mathbb{E}\alphab_{li}^T\betab_{lj} + 4\frac{N(N-n_2)}{n_2(N-1)}\sum_{i<j}\mathbb{E}\alphab_{li}^T\betab_{lj}~.
	\end{align*}
\end{proposition}

\begin{proof}[Proof of Proposition~\ref{prop:MSE_vr}]
	
	First, for the $A_l$ term, from the proof of Theorem~\ref{theo:mse_batch}, we know that
	\begin{align*}
	A_l &= \left(\frac{N}{n_2} - 1\right)\sum_{ij}\mathbb{E}\alphab_{li}\alphab_{lj} - 2\frac{N(N-n_2)}{n_2(N-1)}\sum_{i<j}\mathbb{E}\alphab_{li}\alphab_{lj}~,
	\end{align*}
	which is the value of $\mathbb{E}\|\Delta V_l\|^2$ for standard SG-MCMC. 
	
	The derivations for $B_l$ and $C_l$ go as follows. For $B_l$, we have
	\begin{align*}
	&\mathbb{E}\left(\frac{N}{n_2}z_i - \frac{N}{n_1}b_i\right)\left(\frac{N}{n_2}z_j - \frac{N}{n_1}b_j\right) \\
	=& \mathbb{E}\left(\frac{N^2}{n_2^2}z_iz_j + \frac{N^2}{n_1^2}b_ib_j - \frac{N^2}{n_1n_2}z_ib_j - \frac{N^2}{n_1n_2}b_iz_j\right) \\
	=& \mathbb{E}\left(\frac{N^2}{n_2^2}z_iz_j + \frac{N^2}{n_1^2}b_ib_j - 2\right)~.
	\end{align*}
	
	If $i = j$,
	\begin{align*}
	&\mathbb{E}\left(\frac{N^2}{n_2^2}z_iz_j + \frac{N^2}{n_1^2}b_ib_j - 2\right) 
	= \mathbb{E}\left(\frac{N^2}{n_2^2}z_i + \frac{N^2}{n_1^2}b_i - 2\right) \\
	&= \frac{N}{n_2} + \frac{N}{n_1} - 2~.
	\end{align*}
	
	If $i \neq j$,
	\begin{align*}
	&\mathbb{E}\left(\frac{N^2}{n_2^2}z_iz_j + \frac{N^2}{n_1^2}b_ib_j - 2\right) \\
	=&\frac{N^2}{n_2^2} \frac{n_2}{N} \frac{n_2-1}{N-1} + \frac{N^2}{n_1^2} \frac{n_1}{N}\frac{n_1-1}{N-1} - 2 \\
	=& \frac{N}{n_2}\frac{n_2-1}{N-1} + \frac{N}{n_1}\frac{n_1-1}{N-1} - 2~.
	\end{align*}
	
	\begin{align*}
	B_l &= \left(\frac{N}{n_2} + \frac{N}{n_1} - 2\right)\left(\sum_i \mathbb{E}\betab_i^2\right) \\
	&~~+ 2\left(\frac{N}{n_2}\frac{n_2-1}{N-1} + \frac{N}{n_1}\frac{n_1-1}{N-1} - 2\right)\left(\sum_{i<j}\mathbb{E}\betab_{li}\betab_{lj}\right) \\
	&= \left(\frac{N}{n_2} + \frac{N}{n_1} - 2\right)\sum_{ij} \mathbb{E}\betab_{li}\betab_{lj} \\
	&~~+ 2\left(\frac{N}{n_2}\frac{n_2-1}{N-1} + \frac{N}{n_1}\frac{n_1-1}{N-1} - \frac{N}{n_2} - \frac{N}{n_1}\right)\sum_{i<j}\mathbb{E}\betab_{li}\betab_{lj} \\
	&= \left(\frac{N}{n_2} + \frac{N}{n_1} - 2\right)\sum_{ij} \mathbb{E}\betab_{li}\betab_{lj} \\
	&~~- 2\left(\frac{N(N-n_2)}{n_2(N-1)} + \frac{N(N-n_1)}{n_1(N-1)}\right)\sum_{i<j}\mathbb{E}\betab_{li}\betab_{lj}
	\end{align*}
	
	Similarly, for $C_l$, we have
	\begin{align*}
	C_l &= 2\mathbb{E}\sum_i\sum_j\alphab_{li}\betab_{lj}\left(\frac{N}{n_2}z_j - \frac{N^2}{n_2^2}z_iz_j \right.\\
	&~~\left.- \frac{N}{n_1}b_j + \frac{N^2}{n_1n_2}z_ib_j\right) \\
	&=2\mathbb{E}\sum_i \alphab_{li} \betab_{lj}\left(1 - \frac{N}{n_2}\right) \\
	&~~+ 2\mathbb{E}\sum_{i\neq j}\alphab_{li}\betab_{lj}\left(1 - \frac{N}{n_2}\frac{n_2-1}{N-1}\right) \\
	&= 2\sum_{ij}\mathbb{E}\alphab_{li}\betab_{lj}\left(1 - \frac{N}{n_2}\right) \\
	&~~+ 4\frac{N(N-n_2)}{n_2(N-1)}\sum_{i \leq j}\mathbb{E}\alphab_{li}\betab_{lj}~.
	\end{align*}
	This completes the proof.
	
\end{proof}

The following derivations verify an intuition: with larger minibatch size $n_1$, we can get smaller MSEs. This is not directly relevant to the proof of Theorem~\ref{theo:main}, readers can choose to skip this part without affecting the flow of the proof.

To show that, let's first look at the term $B_l + C_l$ defined above. We have that

\begin{align*}
B_l &+ C_l = \left(\frac{N}{n_1} + \frac{N}{n_2} - 2\right)\sum_{ij}\mathbb{E}\betab_{li}\betab_{lj} \\
&- 2\sum_{i < j}\mathbb{E}\betab_{li}\betab_{lj}\left(\frac{N(N-n_2)}{n_2(N-1)} + \frac{N(N-n_1)}{n_1(N-1)}\right) \\
&+ 2\left(1 - \frac{N}{n_2}\right)\sum_{ij}\mathbb{E}\alphab_{li}\betab_{lj} + 4\frac{N(N-n_2)}{n_2(N-1)}\sum_{i<j}\mathbb{E}\alphab_{li}\betab_{lj}~.
\end{align*}

When $n_1 = N$, the case of using the whole data to calculate {\em old} gradient $\tilde{g}$ \cite{DubeyRPSX:nips16}, we have
\begin{align*}
&B_l + C_l \\
=& \left(\frac{N}{n_2} - 1\right)\sum_{ij}\mathbb{E}\betab_{li}\betab_{lj} - 2\frac{N(N-n_2)}{n_2(N-1)}\sum_{i < j}\mathbb{E}\betab_{li}\betab_{lj} \\
&+ 2\left(1 - \frac{N}{n_2}\right)\sum_{ij}\mathbb{E}\alphab_{li}\betab_{lj} + 4\frac{N(N-n_2)}{n_2(N-1)}\sum_{i<j}\mathbb{E}\alphab_{li}\betab_{lj} \\
=&\left(\frac{N}{n_2} - 1\right)\sum_{ij}\left(\mathbb{E}\betab_{li}\betab_{lj} - 2\mathbb{E}\alphab_{li}\betab_{lj}\right) \\
&~~+ 2\frac{N(N-n_2)}{n_2(N-1)}\sum_{i < j}\left(2\mathbb{E}\alphab_{li}\betab_{lj} - \mathbb{E}\betab_{li}\betab_{lj}\right) \\
=& \frac{(N-n_2)N^2}{n_2}\left[\frac{1}{N^2}\sum_{ij}\left(\mathbb{E}\betab_{li}\betab_{lj} - 2\mathbb{E}\alphab_{li}\betab_{lj}\right)\right. \\
&\left.+ \frac{2}{N(N-1)}\sum_{i<j}\left(2\mathbb{E}\alphab_{li}\betab_{lj} - \mathbb{E}\betab_{li}\betab_{lj}\right)\right] \triangleq M_{BC}
\end{align*}

When $n_1 \neq N$, we have
\begin{align*}
&B_l + C_l \\
=& M_{BC} + \frac{N-n_1}{n_1}\sum_{ij}\mathbb{E}\betab_{li}\betab_{lj} - 2\frac{N(N-n_1)}{n_1(N-1)}\sum_{i<j}\mathbb{E}\betab_{li}\betab_{lj} \\
=& M_{BC} + \frac{(N-n_1)N^2}{n_1}\left[\frac{1}{N^2}\sum_{ij}\mathbb{E}\betab_{li}\betab_{lj}\right.\\
&~~\left.- \frac{2}{N(N-1)}\sum_{i<j}\mathbb{E}\betab_{li}\betab_{lj}\right]~.
\end{align*}

According to Lemma~\ref{lem:exp_alpha}, we have that $\left[\frac{1}{N^2}\sum_{ij}\mathbb{E}\betab_{li}\betab_{lj} - \frac{2}{N(N-1)}\sum_{i<j}\mathbb{E}\betab_{li}\betab_{lj}\right] \geq 0$. As a result, the value of $B_l + C_l$ in the case of $n_1 \neq N$ is larger than that in the case of $n_1 = N$, resulting in a larger MSE bound. 

Now it is ready to prove Theorem~\ref{theo:main}.

\begin{proof}[Proof of Theorem~\ref{theo:main}]
	
	Note that term $A_l$ corresponds to the $\mathbb{E}\Delta V_l$ term in standard SG-MCMC, where no variance reduction is performed. As a result, in order to prove that vrSG-MCMC induces a lower MSE bound, what remains to be shown is to prove $B_l + C_l \leq 0$.
	
	First, let us simplify term $C_l$, which results in:
	{\small\begin{align*}
		&C_l = 2\left(1 - \frac{N}{n_2}\right)\sum_{ij}\mathbb{E}\alphab_{li}\betab_{lj} + 4\frac{N(N-n_2)}{n_2(N-1)}\sum_{i<j}\mathbb{E}\alphab_{li}\betab_{lj} \\
		=& 2\left(1 - \frac{N}{n_2}\right)\sum_{ij}\mathbb{E}\betab_{li}\betab_{lj} + 4\frac{N(N-n_2)}{n_2(N-1)}\sum_{i<j}\mathbb{E}\betab_{li}\betab_{lj} \\
		&+ 2\left(1 - \frac{N}{n_2}\right)N\cdot O(mh) + 4\frac{N(N-n_2)}{n_2(N-1)}\frac{N(N-1)}{2}\cdot O(mh) \\
		=& 2\left(1 - \frac{N}{n_2}\right)\sum_{ij}\mathbb{E}\betab_{li}\betab_{lj} + 4\frac{N(N-n_2)}{n_2(N-1)}\sum_{i<j}\mathbb{E}\betab_{li}\betab_{lj} + O(mh)~,
		\end{align*}}
	where the second equality is obtained by applying the independence property of $\alphab_{li}$ and $\betab_{lj}$, as well as the result from Lemma~\ref{lem:alpha_beta}.
	Consequently, $B_l + C_l$ can be simplified as
	
	\begin{align*}
	B_l + C_l &= \left(\frac{N}{n_1} - \frac{N}{n_2}\right)\sum_{ij}\mathbb{E}\betab_{li}\betab_{lj} - 2\left(\frac{N(N-n_2)}{n_2(N-1)} \right.\\
	&~~\left.- \frac{N(N-n_1)}{n_1(N-1)}\right)\sum_{i<j}\mathbb{E}\betab_{li}\betab_{lj} + O(mh)
	\end{align*}
	
	By substituting the above formula in to the MSE bound in Lemma~\ref{lem:biasmse}, we have that:
	\begin{align*}
	\mathbb{E}\left(\hat{\phi}_L - \bar{\phi}\right)^2 = O \left(\frac{A_M}{L} + \frac{1}{Lh} + h^{2K} + \frac{mh}{L} - \frac{\lambda_M}{L}\right)~.
	\end{align*}
	
	To further simplify the $B_l+C_l$ term, we have
	\begin{align*}
	&B_l + C_l = O(mh) + \frac{N^3(n_2 - n_1)}{n_1n_2}\frac{1}{N^2}\sum_{ij}\mathbb{E}\betab_{li}\betab_{lj} \\
	-& 2\frac{N^2(N-n_2)n_1 - N^2(N - n_1)n_2}{n_1n_2} \frac{1}{N(N - 1)}\sum_{i<j}\mathbb{E}\betab_{li}\betab_{lj} \\
	&= \frac{N^3(n_2 - n_1)}{n_1n_2}\left(\frac{1}{N^2}\sum_{ij}\mathbb{E}\betab_{li}\betab_{lj} \right.\\
	&~~\left.- 2\frac{1}{N(N - 1)}\sum_{i<j}\mathbb{E}\betab_{li}\betab_{lj} \right) + O(mh)
	\end{align*}
	
	According to Lemma~\ref{lem:exp_alpha}, $\left[\frac{1}{N^2}\sum_{ij}\mathbb{E}\betab_{li}\betab_{lj} - 2\frac{1}{N(N - 1)}\sum_{i<j}\mathbb{E}\betab_{li}\betab_{lj} \right] \geq 0$. Consequently, we have $B_l + C_l \leq 0$ up to an order of $O(mh)$. This completes the proof of $\lambda_M \geq 0$.
	
\end{proof}

\section{Discussion of the Theoretical Results of Dubey {\it et al.}\! 2016}
\label{app:discuss}

\cite{DubeyRPSX:nips16} proved the following MSE bound for SVRG-LD, by extending results of the standard SG-MCMC \cite{ChenDC:NIPS15}:

\begin{align}\label{eq:svrgld}
&\mathbb{E}\left(\hat{\phi}_L - \bar{\phi}\right)^2 = O\left(\frac{N^2\min\{2\sigma^2, m^2(D^2h^2\sigma^2+hd)\}}{nL} \right.\nonumber\\
&\left.~~+ \frac{1}{Lh} + h^2\right)~,
\end{align}
where $(d, D, \sigma)$ are constants related to the data and the true posterior. Using similar techniques (shown in the paper), the MSE bound for SGLD is given by
\begin{align}\label{eq:sgld}
\mathbb{E}\left(\hat{\phi}_L - \bar{\phi}\right)^2 = O\left(\frac{N^2\sigma^2}{nL} + \frac{1}{Lh} + h^2\right)~.
\end{align}

From the proof of their theorem (eq.~13 in their appendix), we note that the constant ``2'' inside the ``min'' in \eqref{eq:svrgld} is not negligible when comparing to the bound for SGLD. As a result, the bound associated with this term is strictly larger than the bound for SGLD. This means that to compared with SGLD, the MSE bound for SVRG-LD should be written in the form of 
\begin{align}\label{eq:svrgld1}
\mathbb{E}\left(\hat{\phi}_L - \bar{\phi}\right)^2
= O\left(\frac{N^2m^2(D^2h^2\sigma^2+hd)}{nL} + \frac{1}{Lh} + h^2\right)~.
\end{align}

As a result, the comparison between \eqref{eq:svrgld1} and \eqref{eq:sgld} becomes more complicated, because it now depends on other parameters such as the stepsize. It is thus not clear if SVRG-LD would improve the MSE bound of SGLD.

In contrast, our theoretical results (Theorem~\ref{theo:main}) guarantee an improvement of vrSG-MCMC over the correspond SG-MCMC, which is a stronger result than that in \cite{DubeyRPSX:nips16}.

\section{Additional Experimental Results}\label{app:exp}

%
%

\subsection{Supplemental results on logistic regression and deep learning}

We plot the corresponding results in terms of number of passes through data versus training error/loss in Figure~\ref{fig:fnn_cnn_train} and Figure~\ref{fig:n1_train}.


\begin{figure*}[ht]
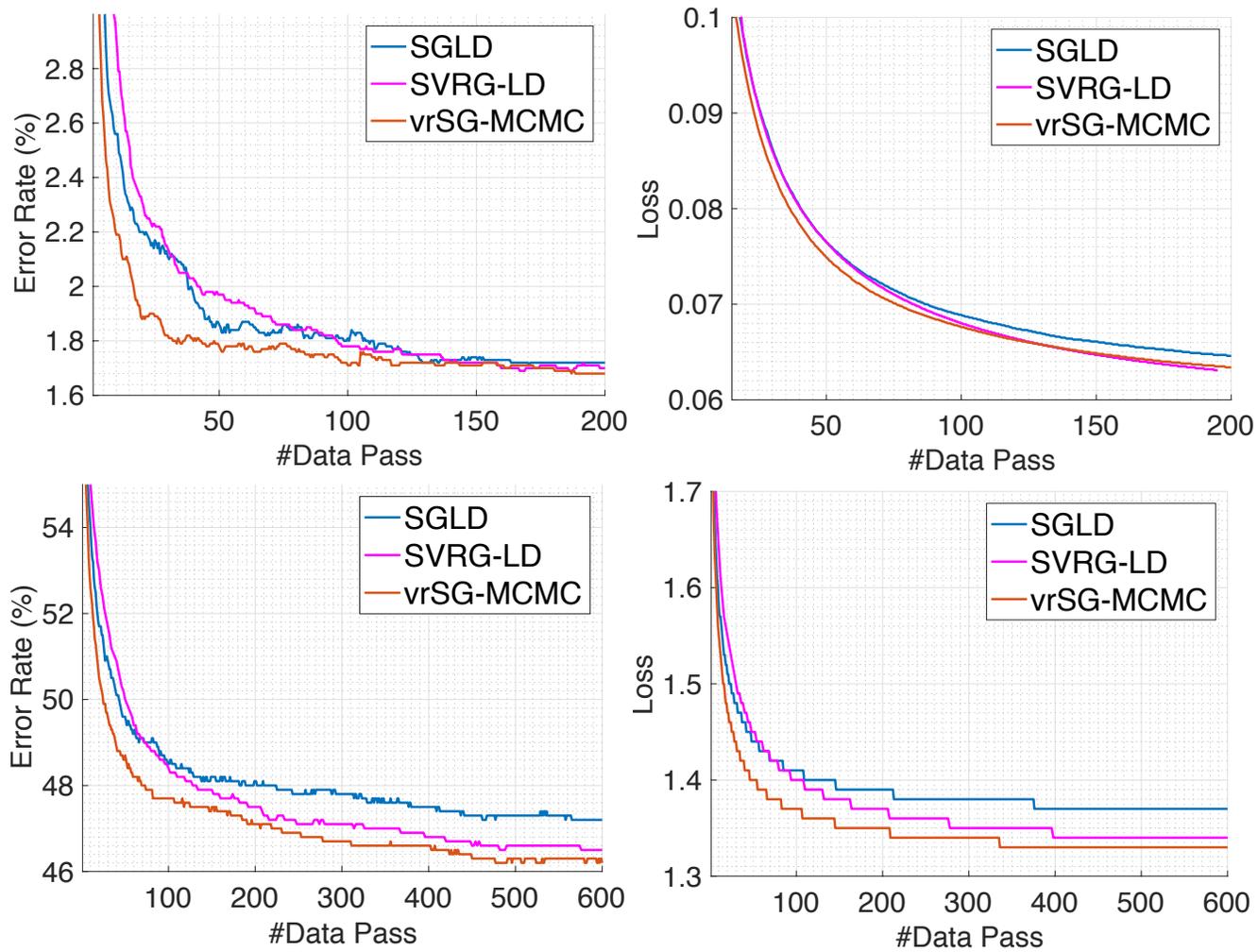

	\begin{center}
		\begin{minipage}{0.49\linewidth}
			\includegraphics[width=\columnwidth]{MNIST_FFNN_Test_Error}
		\end{minipage}
		\begin{minipage}{0.49\linewidth}
			\includegraphics[width=\columnwidth]{MNIST_FFNN_Test_Loss}
		\end{minipage}
		\begin{minipage}{0.49\linewidth}
			\includegraphics[width=\columnwidth]{CIFAR10_FFNN_Test_Error}
		\end{minipage}
		\begin{minipage}{0.49\linewidth}
			\includegraphics[width=\columnwidth]{CIFAR10_FFNN_Test_Loss}
		\end{minipage}
		\caption{Number of passes through data vs. testing error (left) / loss (right) on MNIST (top) and CIFAR-10 (bottom) datasets.}
		\label{fig:fnn1}
	\end{center}
	\vskip -0.2in
\end{figure*} 

\begin{figure*}[ht]
	\begin{center}
		\begin{minipage}{0.49\linewidth}
			\includegraphics[width=\columnwidth]{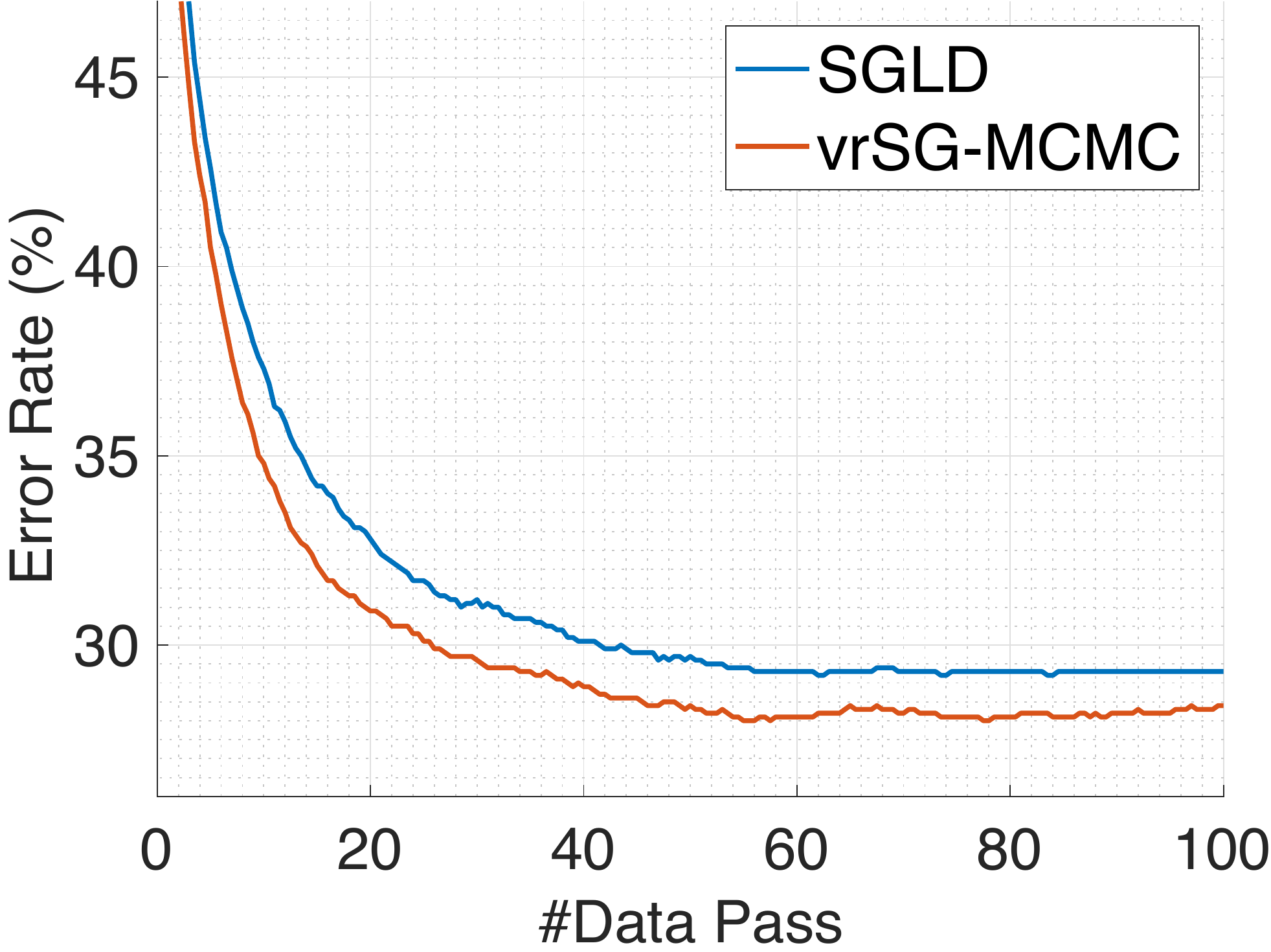}
		\end{minipage}
		\begin{minipage}{0.49\linewidth}
			\includegraphics[width=\columnwidth]{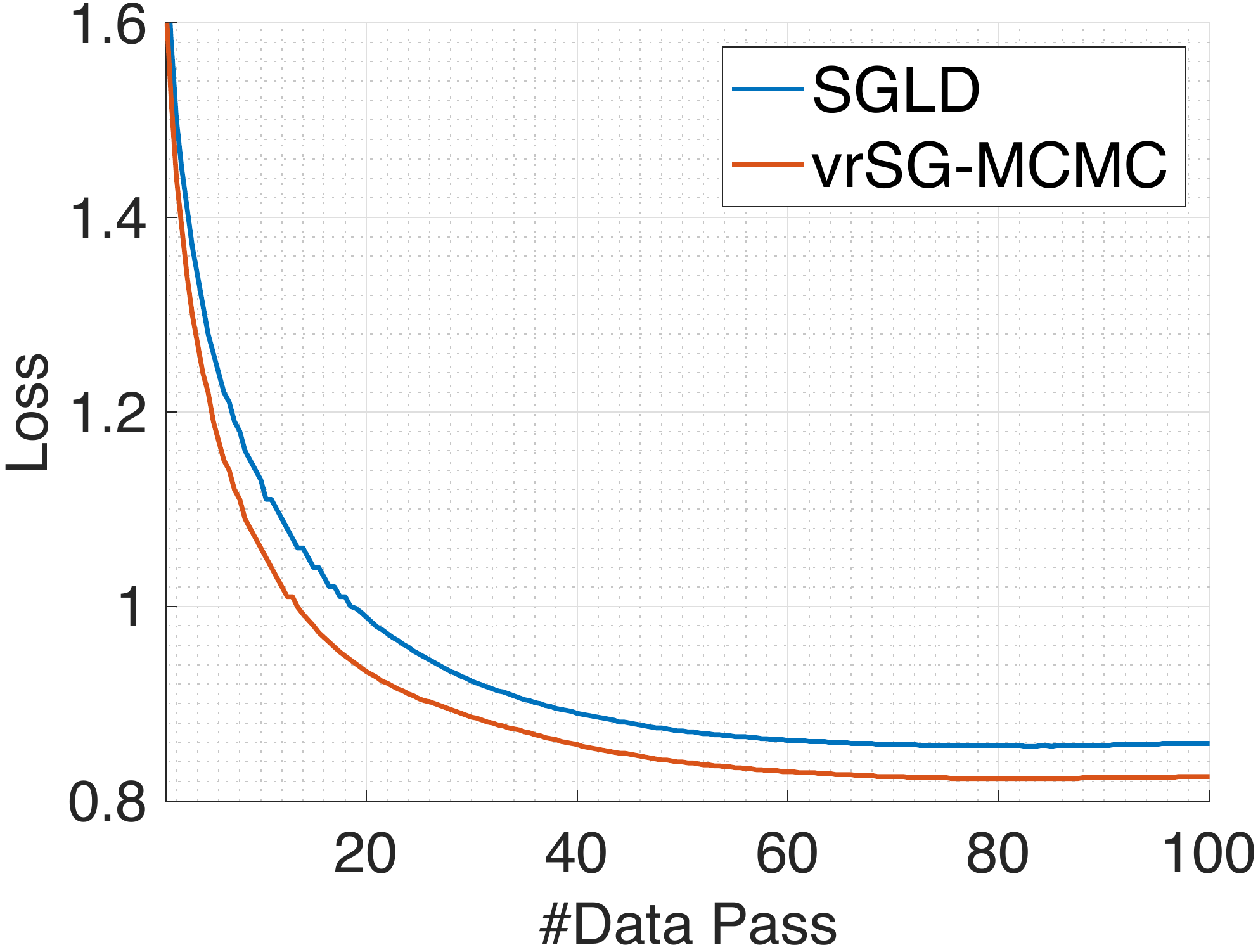}
		\end{minipage}
		\begin{minipage}{0.49\linewidth}
			\includegraphics[width=\columnwidth]{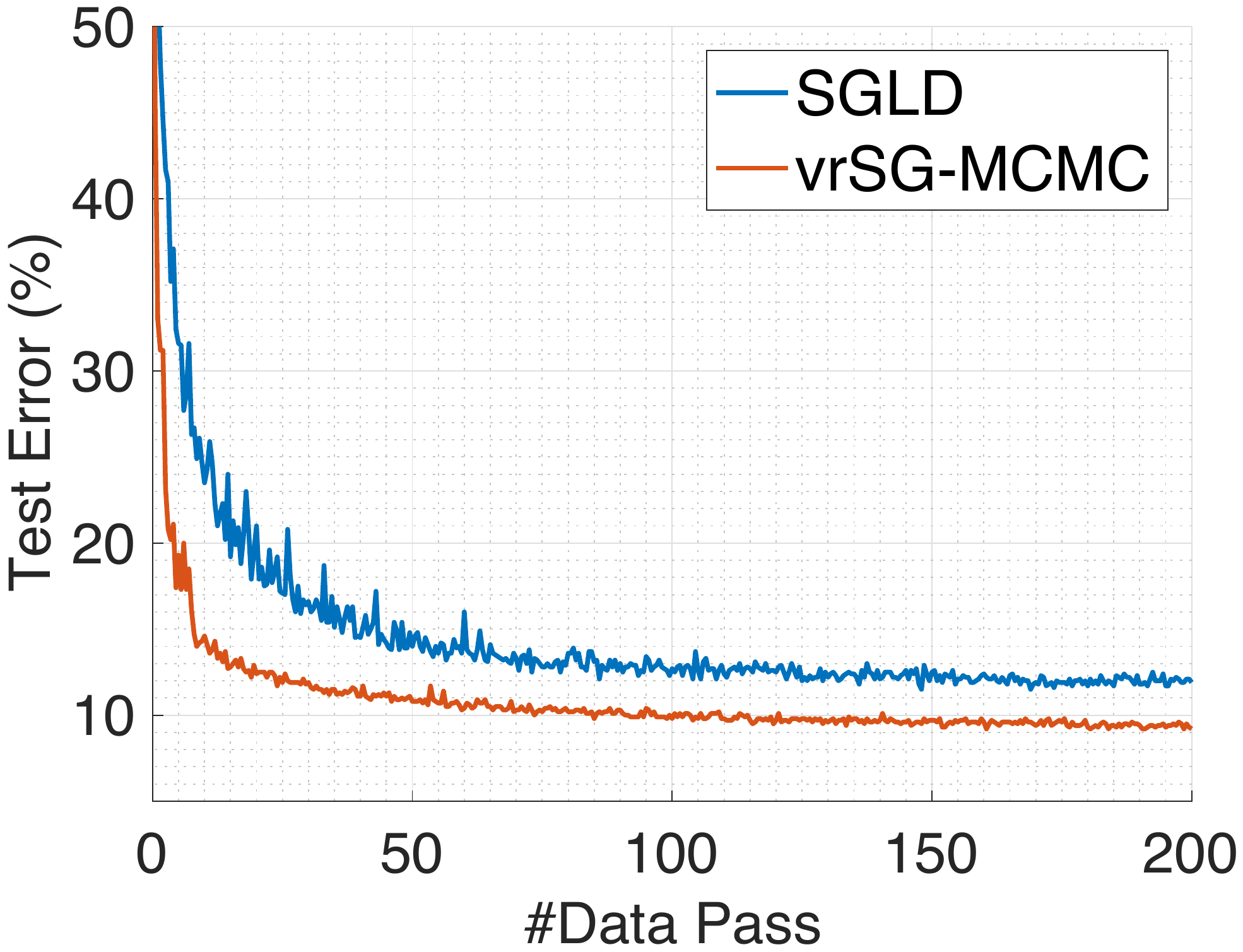}
		\end{minipage}
		\begin{minipage}{0.49\linewidth}
			\includegraphics[width=\columnwidth]{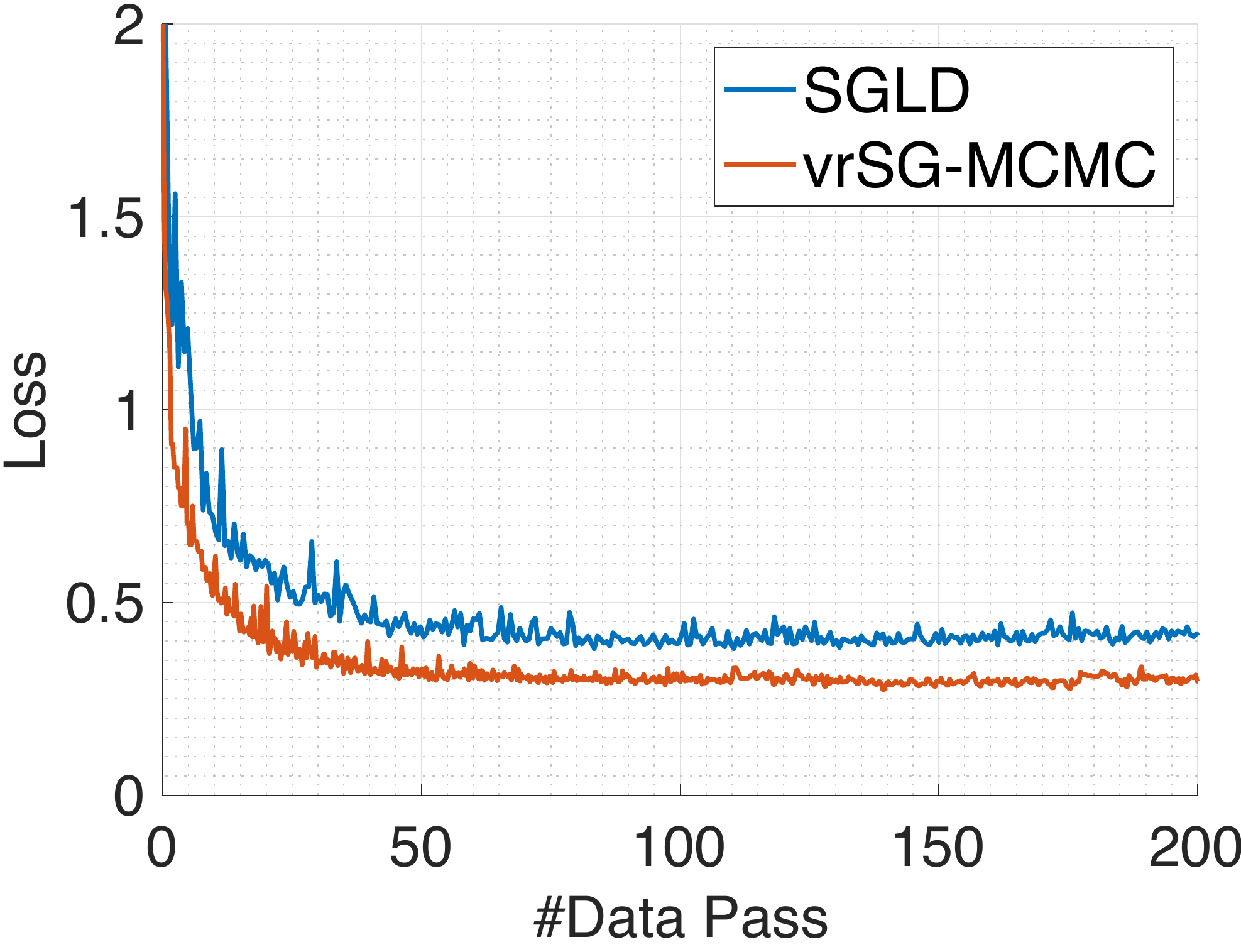}
		\end{minipage}
		\caption{Number of passes through data vs. testing error (left) / loss (right) with CNN-4 (top) and ResNet (bottom) on CIFAR-10.}
		\label{fig:cnn1}
	\end{center}
	\vskip -0.2in
\end{figure*}

\begin{figure*}[!h]
	\centering
	\subfigure[MNIST-FFNN-Train-Error]{\label{fig:a}\includegraphics[width=.49\linewidth]{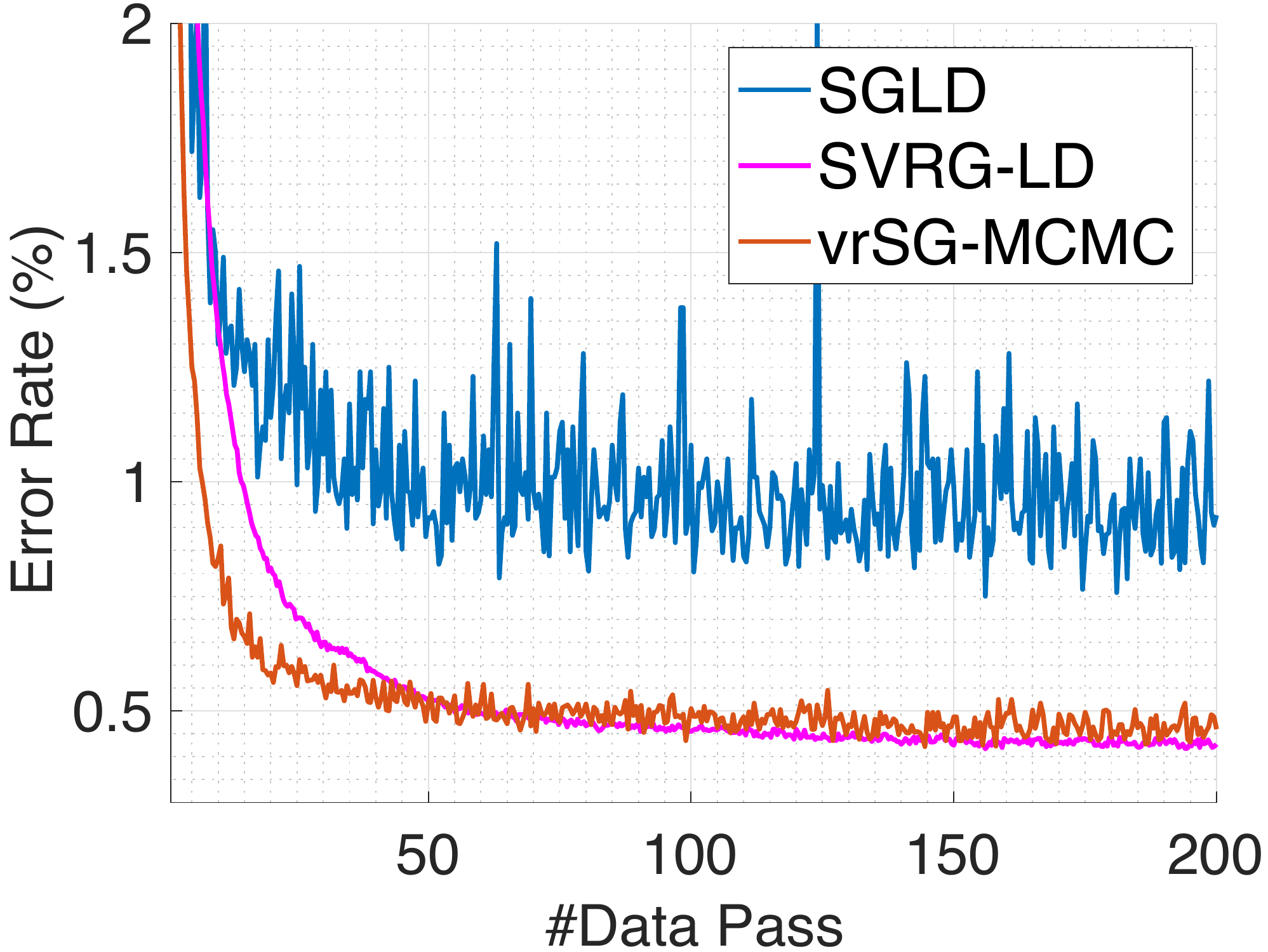}}
	\subfigure[MNIST-FFNN-Train-Loss]{\label{fig:b}\includegraphics[width=.49\linewidth]{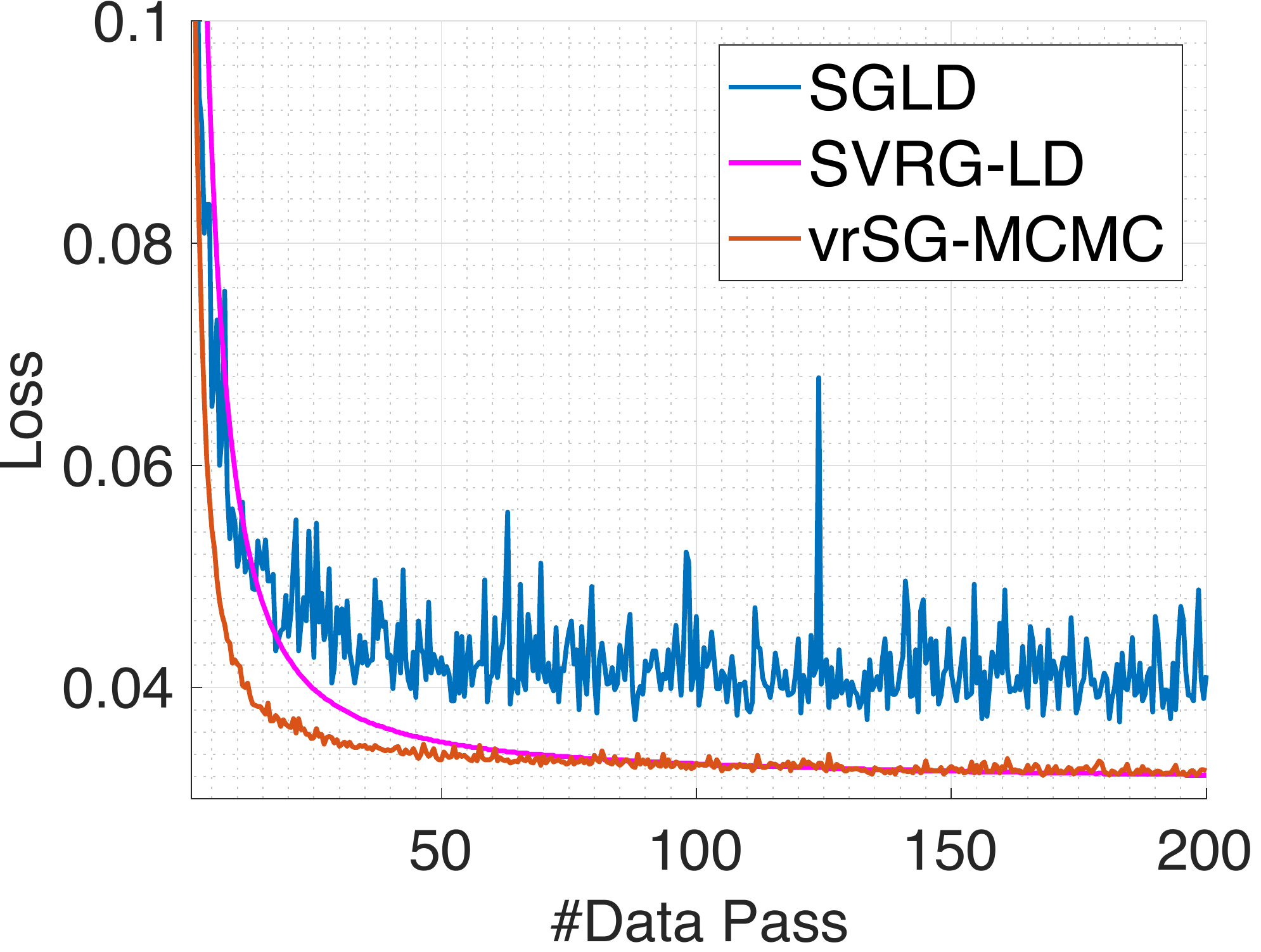}}
	\subfigure[CIFAR10-FFNN-Train-Error]{\label{fig:c}\includegraphics[width=.49\linewidth]{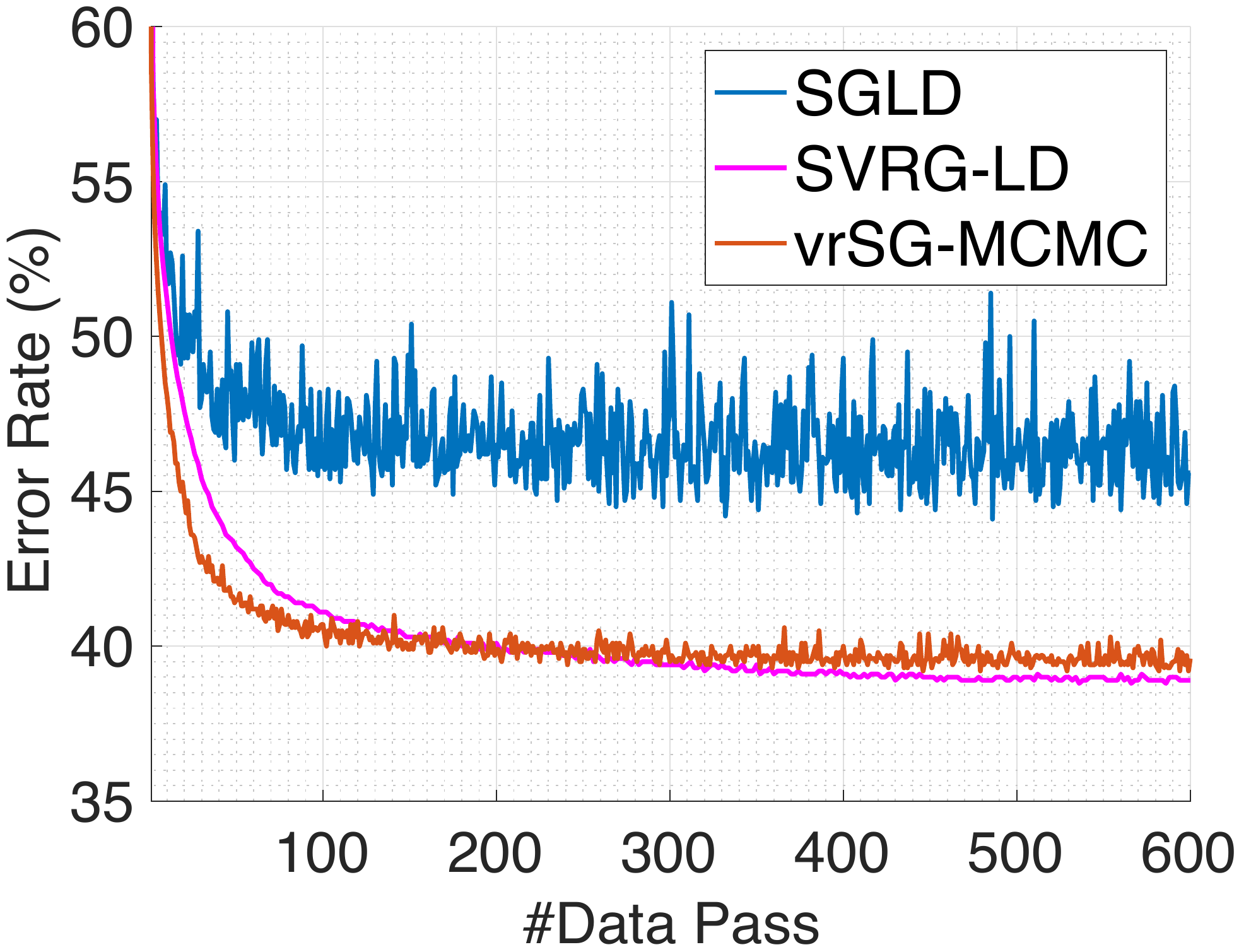}}
	\subfigure[CIFAR10-FFNN-Train-Loss]{\label{fig:d}\includegraphics[width=.49\linewidth]{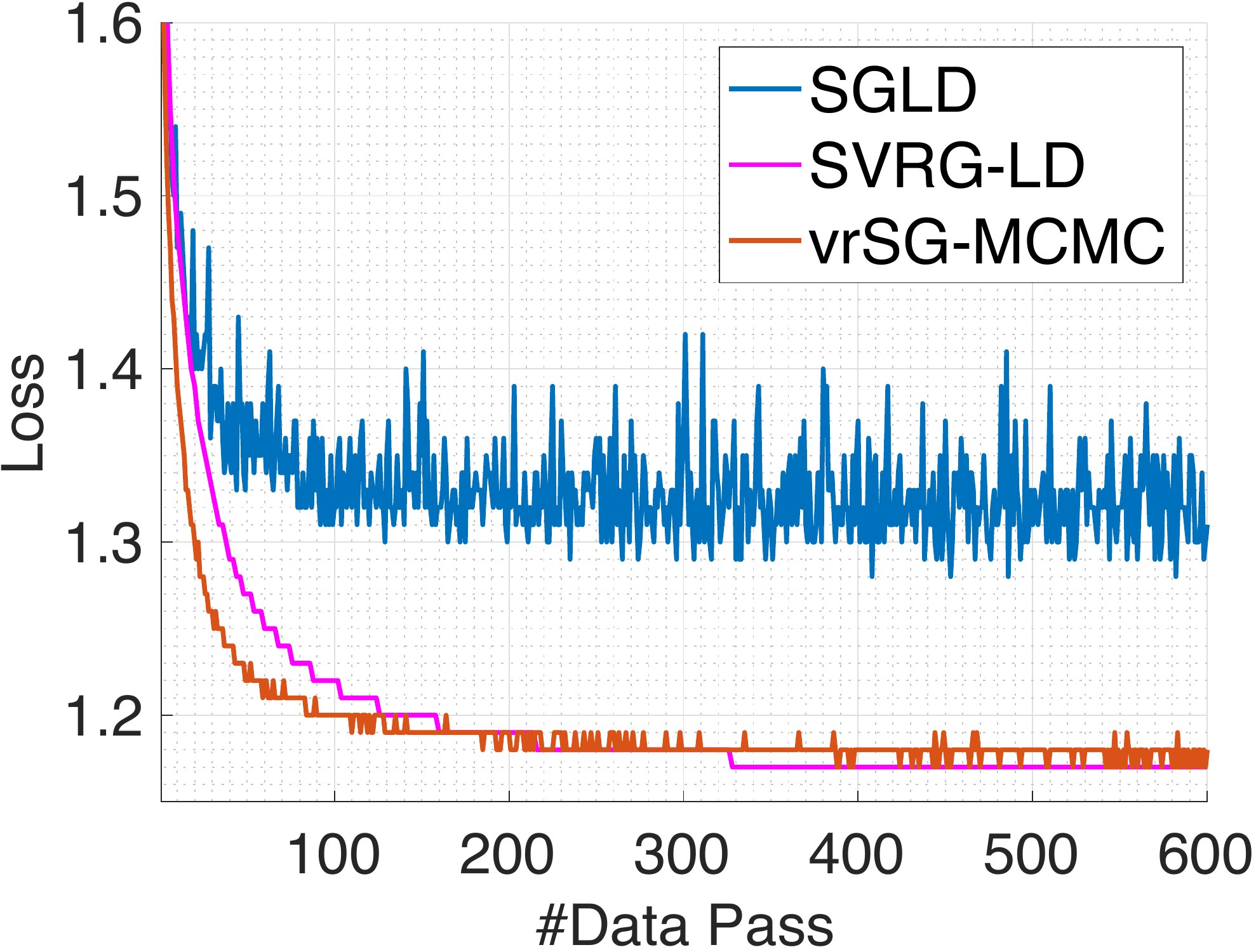}}
	\subfigure[CIFAR10-CNN-L4-Test-Error]{\label{fig:e}\includegraphics[width=.49\linewidth]{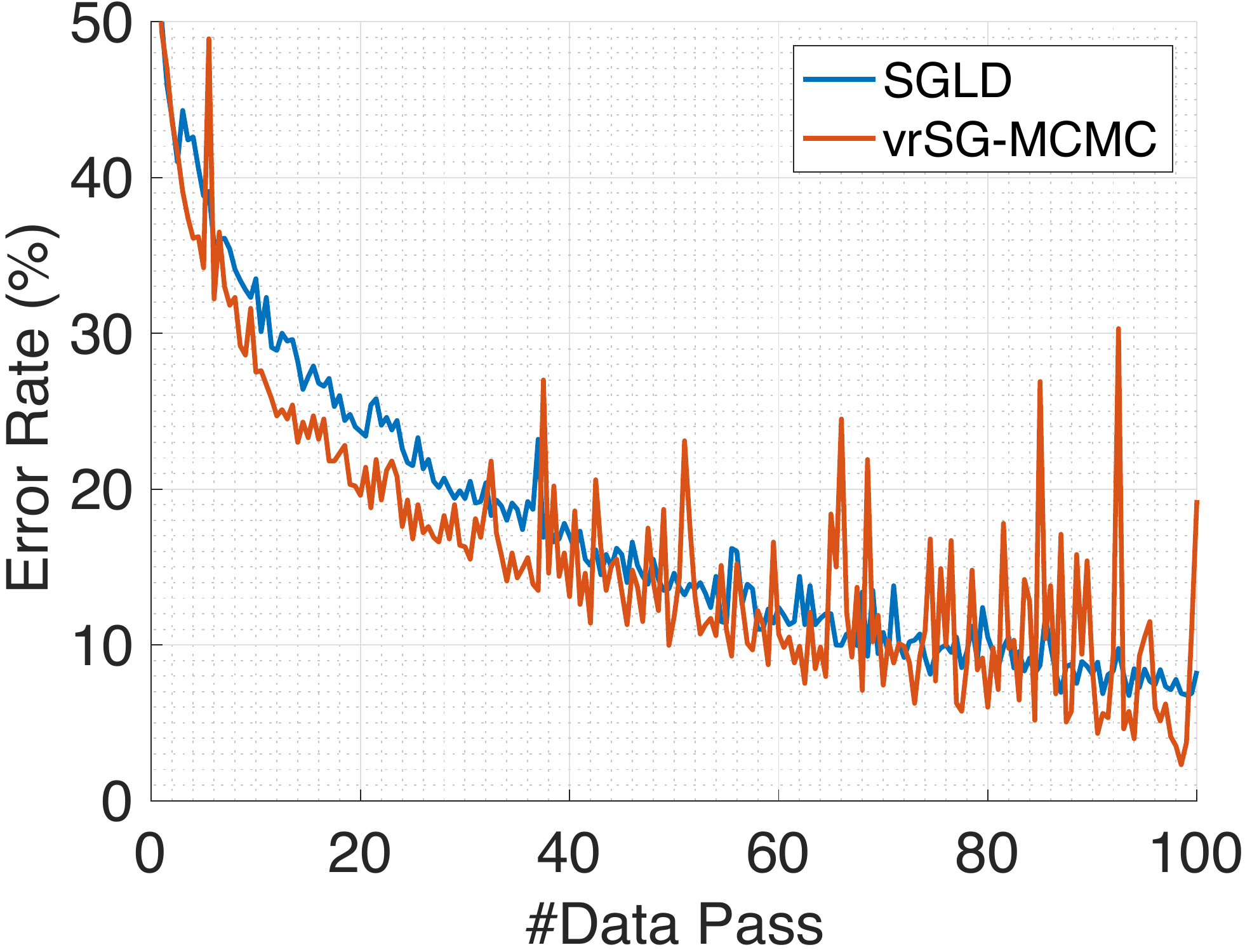}}
	\subfigure[CIFAR10-CNN-L4-Test-Loss]{\label{fig:f}\includegraphics[width=.49\linewidth]{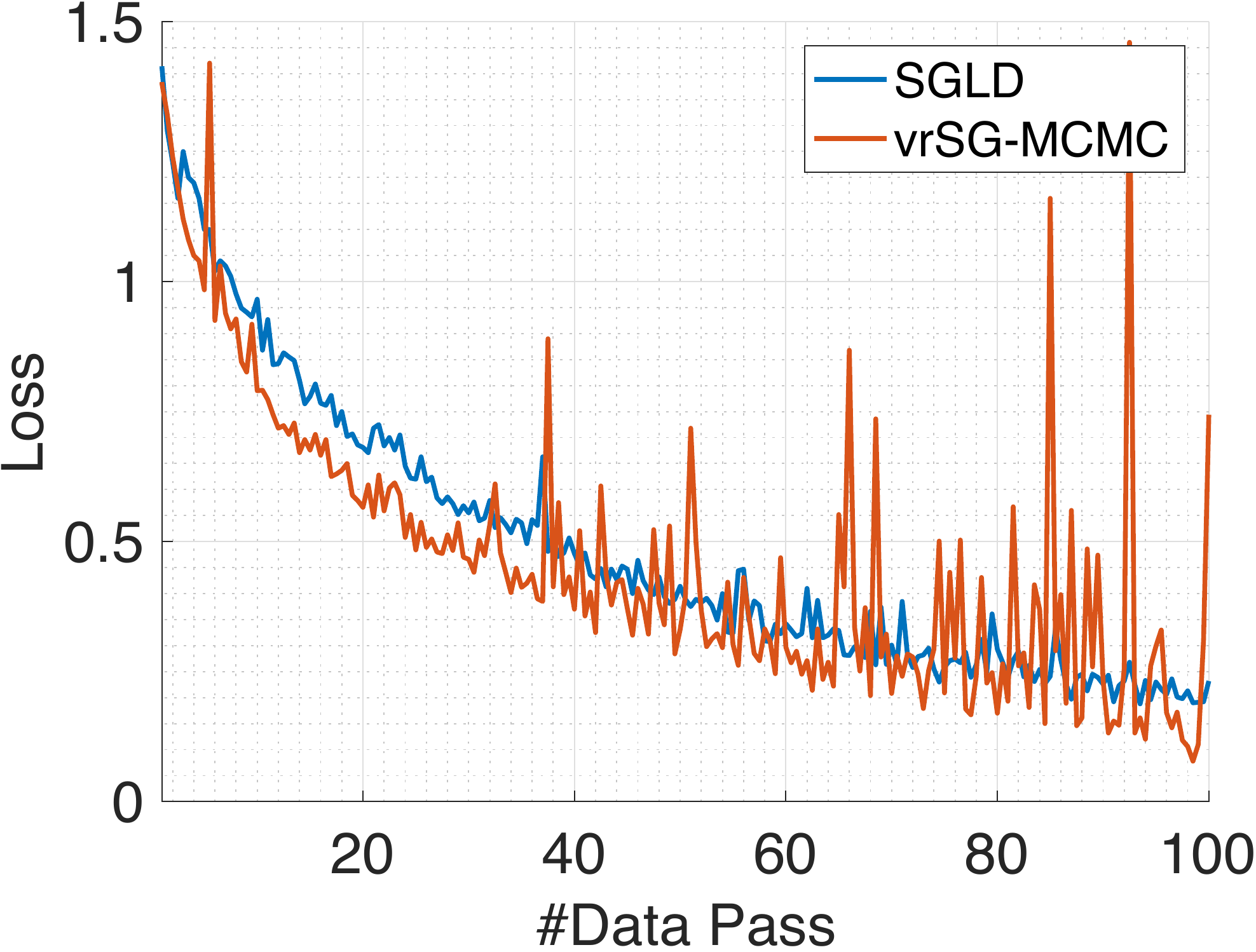}}
	\caption{Number of passes through data vs. training error / loss on MNIST and Cifar-10 datasets.}
	\label{fig:fnn_cnn_train}
\end{figure*}

\begin{figure*}[ht]
	\begin{center}
		\begin{minipage}{0.48\linewidth}
			\includegraphics[width=\columnwidth]{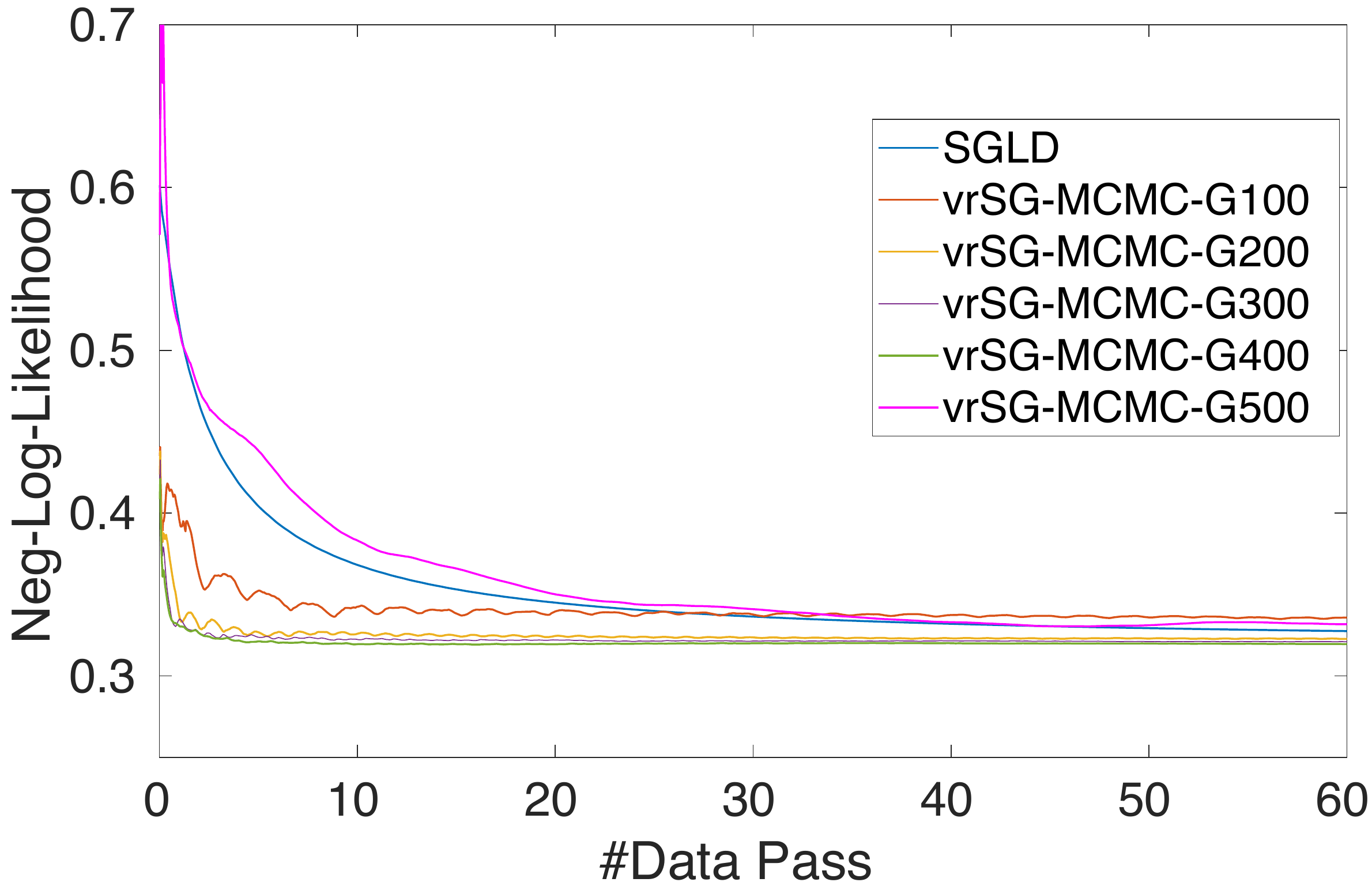}
		\end{minipage} \\\vspace{1cm}
		\begin{minipage}{0.48\linewidth}
			\includegraphics[width=\columnwidth]{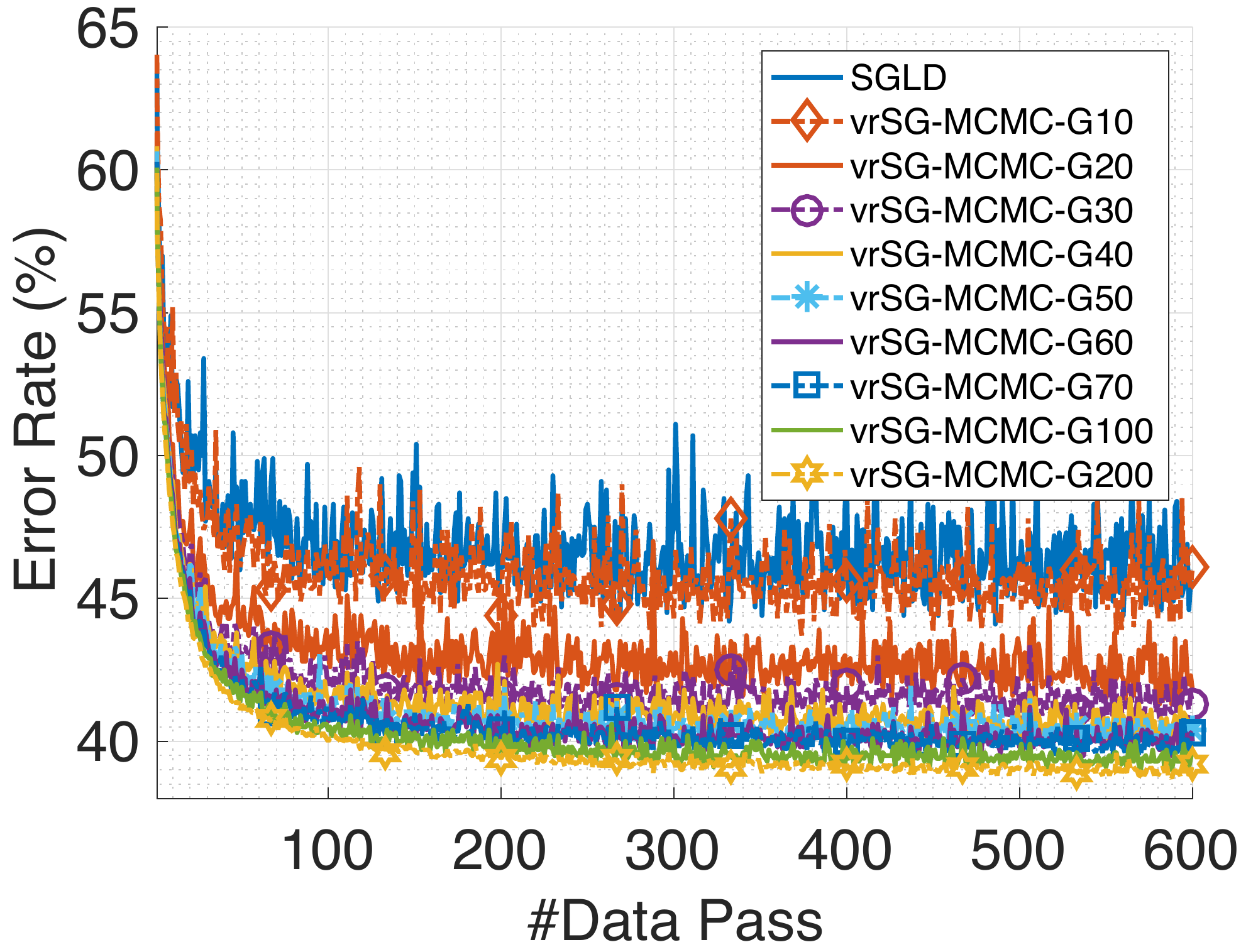}
		\end{minipage}
		\begin{minipage}{0.48\linewidth}
			\includegraphics[width=\columnwidth]{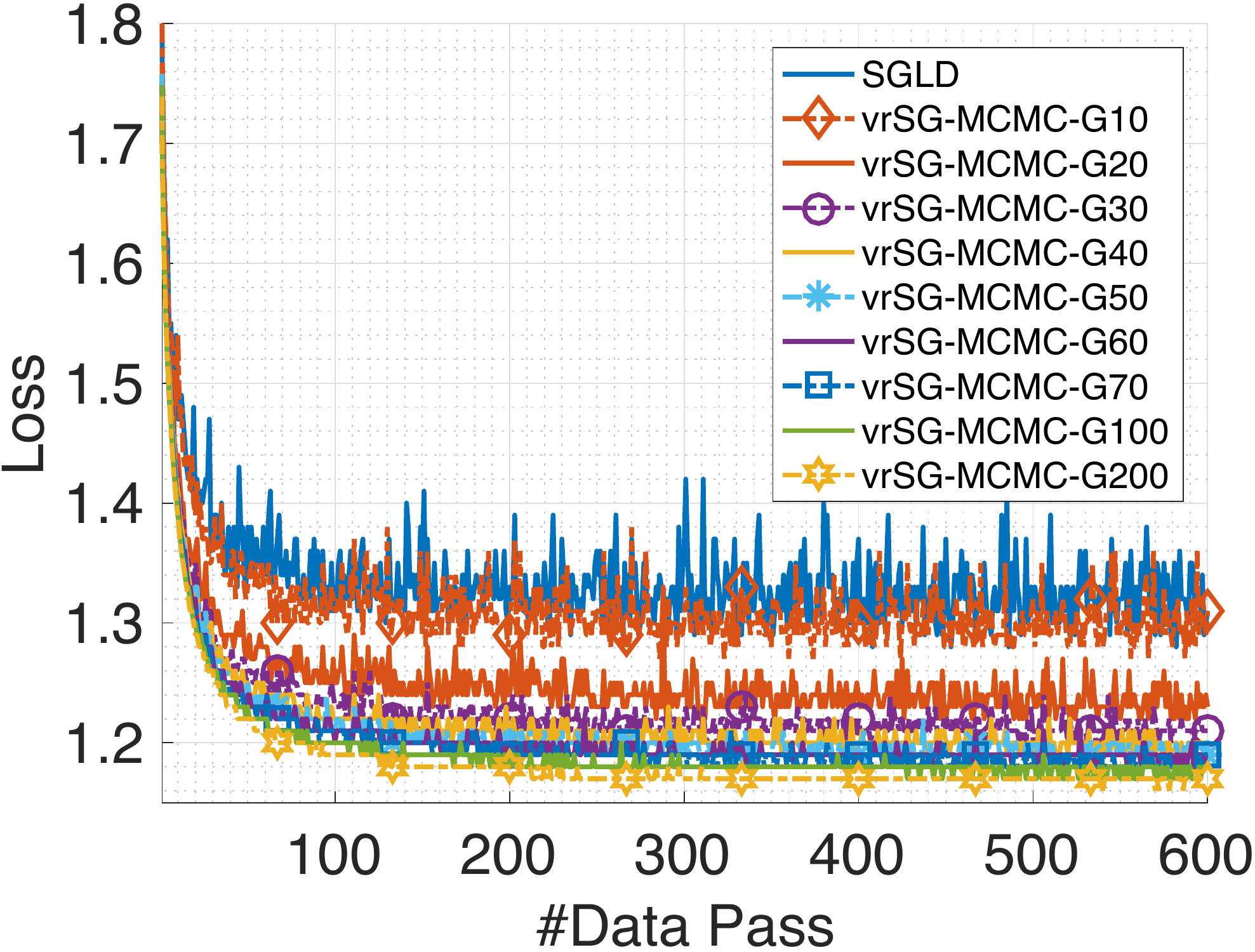}
		\end{minipage}\vskip -0.1in
		\caption{Top: Number of passes through data vs. testing negative log-likelihood on the Pima dataset for Bayesian logistic regression. Bottom: Number of passes through data vs. training errors (left) / loss (right) on the CIFAR-10 dataset. All are with varying $n_1$ values.}
		\label{fig:n1_train}
	\end{center}
\end{figure*}

\end{document}